\documentclass[a4paper,10pt]{article}
\usepackage{times}
\usepackage[margin=2cm]{geometry}

\usepackage{amssymb,amsmath}
\usepackage{amsthm}
\usepackage{graphicx}
\usepackage{xcolor}

\makeatletter
\newtheorem*{rep@theorem}{\rep@title}
\newcommand{\newreptheorem}[2]{%
\newenvironment{rep#1}[1]{%
 \def\rep@title{#2 \ref{##1}}%
 \begin{rep@theorem}}%
 {\end{rep@theorem}}}
\makeatother

\theoremstyle{plain}
\newtheorem{theorem}{Theorem}
\newtheorem{proposition}{Proposition}

\theoremstyle{definition}
\newtheorem{definition}{Definition}
\newtheorem{remark}{Remark}

\newtheorem{lemma}{Lemma}

\newreptheorem{theorem}{Theorem}
\newreptheorem{proposition}{Proposition}

\usepackage{hyperref}

\usepackage{graphicx} % more modern
\graphicspath{{images/}}
\usepackage{subfigure}

\usepackage{natbib}

\usepackage{algorithm}
\usepackage[noend]{algorithmic}

\usepackage{booktabs}
\usepackage{bbm}
\usepackage{xspace}

\usepackage{mindgap}
\usepackage{mindgap_fig}

\author{Eugene Ndiaye, Olivier Fercoq, Alexandre Gramfort, Joseph Salmon \\
\small CNRS LTCI, T\'el\'ecom ParisTech, Universit\'e Paris-Saclay\\
\small 46 rue Barrault, 75013, Paris, France\\}
\date{}

\providecommand{\keywords}[1]{\textbf{\textit{Keywords ---}} #1}

\title{GAP Safe Screening Rules for \SGL}

\begin{document}
\maketitle

\vskip 0.3in

\begin{abstract}
In high dimensional settings, sparse structures are crucial for efficiency, either in term of memory, computation or performance. In some contexts, it is natural to handle more refined structures than pure sparsity, such as for instance group sparsity. \SGL has recently been introduced in the context of linear regression to enforce sparsity both at the feature level and at the group level. We adapt to the case of \SGL recent safe screening rules that discard early in the solver irrelevant features/groups. Such rules have led to important speed-ups for a wide range of iterative methods. Thanks to dual gap computations, we provide new safe screening rules for \SGL and show significant gains in term of computing time for a coordinate descent implementation.
\end{abstract}
%!TEX root = ../icml.tex

\section{Introduction}
Sparsity is a critical property for the success of regression methods, especially in high dimension. Often, group (or block) sparsity is helpful when some known group structure needs to be enforced. This is for instance the case in multi-task learning \citep{Argyriou_Evgeniou_Pontil08} or multinomial logistic regression \citep[Chapter 3]{Buhlmann_vandeGeer11}. In the multi-task setting, the group structure appears natural since one aims at jointly recovering signals whose supports are shared. In this context, sparsity and group sparsity are generally obtained by adding a regularization term to the data-fitting: $\ell_1$ norm for simple sparsity and $\ell_{1,2}$ for group sparsity.

Along with recent works on hierarchical regularization \cite{Jenatton_Mairal_Obozinski_Bach11,Sprechmann_Ramirez_Sapiro_Eldar11, Simon_Friedman_Hastie_Tibshirani13} have focused on a specific case: the \SGL. This method is the solution of a (convex) optimization program with a regularization term that is a convex combination of the two aforementioned norms, enforcing sparsity and group sparsity at the same time.

When using such advanced regularizations, the computational burden can be heavy particularly in high dimension. Yet, it can be significantly reduced if one can exploit the fact that the solution of the optimization problem is sparse. Following the seminal paper on ``safe screening rules'' \citep{ElGhaoui_Viallon_Rabbani12}, many contributions have investigated such strategies \citep{Xiang_Xu_Ramadge11,Bonnefoy_Emiya_Ralaivola_Gribonval14, Bonnefoy_Emiya_Ralaivola_Gribonval15,Wang_Ye14}. These so called safe screening rules compute some tests on dual feasible points to eliminate primal variables whose coefficients are guaranteed to be zero in the exact solution. Still, the computation of a dual feasible point can be challenging when the regularization is more complex than $\ell_1$ or $\ell_{1,2}$ norms. This is the case for the \SGL as it is not straightforward to characterize efficiently if a dual point is feasible or not~\citep{Wang_Ye14}. Hence, an efficient computation of the associated dual norm is required. This is all the more challenging that a naive implementation computing the dual norm associated to the \SGL is very expensive (it is quadratic with respect to the groups dimensions).

Here, we propose efficient dynamic safe screening rules (\ie rules that perform screening as the algorithm proceeds) for the \SGL. More precisely, we elaborate on refinements called GAP safe rules relying on dual gap computations. Such rules have been recently introduced for the Lasso
in \citet{Fercoq_Gramfort_Salmon15} and extended to various tasks in
\citet{Ndiaye_Fercoq_Gramfort_Salmon15}. We propose a natural extension of GAP safe rules to handle the \SGL case.
Moreover, we link the \SGL penalties to the $\epsilon$-norm in \cite{Burdakov88}. We adapt an algorithm introduced in \cite{Burdakov_Merkulov01} to efficiently compute the required
dual norms and highlight geometrical properties of the problem that give an easier way to characterize a dual feasible point. We incorporate our proposed Gap Safe rules in a block coordinate
descent algorithm and show its practical efficiency in climate prediction tasks where the computation time is demanding.

Note that alternative (unsafe) screening rules, for instance the ``strong rules'' \citep{Tibshirani_Bien_Friedman_Hastie_Simon_Tibshirani12}, have been applied to the Lasso and its simple variants. Moreover, strategies also leveraging dual gap computations have recently been considered in the \textbf{Blitz} algorithm \cite{Jonhson_Guestrin15} to speed up working set methods.

\subsection*{Notation}
For any integer $d\in\bbN$, we denote by $[d]$ the set $\{1, \ldots, d\}$. Our
observation vector is $y \in \bbR^n$ and the design matrix $X=
[X_1,\ldots, X_p ] \in \bbR^{n\times p}$ has $p$ explanatory variables or
features, stored column-wise. The standard Euclidean norm is written
$\|\cdot\|$, the $\ell_1$ norm $\|\cdot\|_1$, the $\ell_\infty$ norm
$\|\cdot\|_\infty$, and the transpose of a matrix $Q$ is denoted by
${Q}^\top$. We also denote $(t)_+=\max(0,t)$.

We consider problems where the vector of parameter $\beta = (\beta_1, \ldots,
\beta_p)^\top$ admits a natural group structure. A group of features is a
subset $g \subset [p]$ and $n_g$ is its cardinality. The set of groups is
denoted by $\mathcal{G}$ and we focus only on non-overlapping groups that form a partition of the set $[p]$. We denote by $\beta_g$ the vector in $\bbR^{n_g}$
which is the restriction of $\beta$ to the indexes in $g$.
% and $0$ otherwise, hence one
% can write $\beta = \sum_{g \in \mathcal{G}} \beta_g$.
We write $[\beta_g]_j$ the $j$-th coordinate of $\beta_g$. We also use the
notation $X_g \in \bbR^{n \times n_g}$ to refer to the sub-matrix of $X$
assembled from the columns with indexes $j \in g$, similarly $[X_g]_j$ is
the $j$-th column of $[X_g]$.

For any norm $\Omega$, $\mathcal{B}_{\Omega}$ refers to the corresponding unit
ball, and $\mathcal{B}$ (resp. $\mathcal{B}_{\infty}$) stands for the
Euclidean (resp. $\ell_\infty$) unit ball.
% The proximal operator of $f$ at
% level $\tau > 0$ is defined by
% % $\mathrm{Prox}_{\tau f}(x) =
% $\argmin_{z \in \bbR^d} f(z) +
% \frac{1}{2 \tau}\norm{x - z}^2$.
% % for any $\tau >0$.
% \jo{I think we nowhere need to introduce the proximal operator, only block and
% normal thresholding is needed}
% For instance, the proximal operator of the Euclidean norm is $\GST{\tau}$,
The soft-thresholding operator (at level $\tau \geq 0$), $\ST{\tau}$, is defined for any $x \in \bbR^d$ by $[\ST{\tau} (x)]_j = \sign(x_j)(|x_j|- \tau)_+$, while the group soft-thresholding (at level $\tau$) is $\GST{\tau}(x) = ( 1 - \tau/ \|x\|)_+ x$. Denoting $\Pi_\mathcal{C}$ the projection on a closed convex set $\mathcal{C}$
% is $\Pi_{C}(x) := \argmin_{z \in C} \norm{x - z}$
yields $\ST{\tau}=\Id-\Pi_{\tau \mathcal{B}_\infty}$.
% We refer to
% \citet{Parikh_Boyd13} for more details on proximal operators.
The sub-differential of a convex function $f : \bbR^d \rightarrow
\bbR$ at $x$ is defined by $ \partial f(x) = \{ z \in \bbR^d: \forall y \in
\bbR^d, f(x) - f(y) \geq {z}^\top (x - y) \}$.
% It is the set of all subgradients.

% We denote $f^*:\bbR^d \rightarrow \bbR$ the Fenchel conjugate of $f$ defined
% for any $z \in \bbR^d$ by $f^* (z) = \sup_{w \in \bbR^d} w^\top z - f(w)$.

For any norm $\Omega$ over $\bbR^d$, $\Omega^D$ is the dual norm of $\Omega$,
and is defined for any $x\in \bbR^d$ by $\Omega^D(x)= \max_{v \in \mathcal{B}_{\Omega}} v^\top x$, \eg $\norm{\cdot}_{1}^{D} = \norm{\cdot}_{\infty}$ and $\norm{\cdot}^{D} = \norm{\cdot}$. We also recall that the sub-differential $\partial \norm{\cdot}_1$ of the $\ell_1$ norm is $\sign(\cdot)$, defined element-wise by
\begin{align}\label{eq:sub-differential_1}
\forall j \in [d], \sign(x)_j =
\begin{cases}
\left\{\sign(x_j) \right\} & {\textnormal{if }}  x_j \neq 0, \\
[-1,1] & \textnormal{if }  {x}_{j} = 0,
\end{cases}
\end{align}
and the sub-differential $\partial \norm{\cdot}$ of the Euclidean norm is
\begin{align}\label{eq:sub-differential_2}
\partial \norm{\cdot}(x)=
\begin{cases}
\left\{ \frac{x}{\norm{x}} \right\} & \textnormal{if }  x \neq 0, \\
\mathcal{B} & \textnormal{if }  x = 0.
\end{cases}
\end{align}

%!TEX root = ../icml.tex

%%%%%%%%%%%%%%%%%%%%%%%%%%%%%%%%%%%%%%%%%%%%%%%%%%%%%%%%%%%%%%%%%%%%%%%%%%%%%%%
%%%%%%%%%%%%%%%%%%%%%%%%%%%%%%%%%%%%%%%%%%%%%%%%%%%%%%%%%%%%%%%%%%%%%%%%%%%%%%%
\section{Convex optimization reminder}
%%%%%%%%%%%%%%%%%%%%%%%%%%%%%%%%%%%%%%%%%%%%%%%%%%%%%%%%%%%%%%%%%%%%%%%%%%%%%%%
%%%%%%%%%%%%%%%%%%%%%%%%%%%%%%%%%%%%%%%%%%%%%%%%%%%%%%%%%%%%%%%%%%%%%%%%%%%%%%%

We first recall the necessary tools for building screening rules, namely the Fermat's first order optimality condition (also called Fermat's rule) and the characterization of the sub-differential of a norm by means of its dual norm.

\begin{proposition}[Fermat's rule]
(\citet[Prop. 26.1]{Bauschke_Combettes11})
For any convex function $ f: \bbR^d \to \bbR,$
 \begin{equation} \label{th:Fermat_rule}
  x^\star \in \argmin_{x \in \bbR^d} f(x) \Longleftrightarrow 0 \in \partial f (x^\star).
 \end{equation}
\end{proposition}

\begin{proposition}%[Subdifferential of a norm]
(\citet[Prop. 1.2]{Bach_Jenatton_Mairal_Obozinski12})
The sub-differential of the norm $\Omega$ at x, denoted $\partial \Omega(x)$, is given by
 \begin{equation} \label{eq:sub-differential_norm} \hspace*{-0.49cm}
  \begin{cases}
   \{z \in \bbR^d: \Omega^D(z) \leq 1 \} = \mathcal{B}_{\Omega^D}
   &\text{ if } x = 0, \\
   \{z \in \bbR^d: \Omega^D(z) = 1 \text{ and } z^\top x = \Omega(x) \}
   &\text{ otherwise}.
  \end{cases}
 \end{equation}
\end{proposition}

%%%%%%%%%%%%%%%%%%%%%%%%%%%%%%%%%%%%%%%%%%%%%%%%%%%%%%%%%%%%%%%%%%%%%%%%%%%%%%%
%%%%%%%%%%%%%%%%%%%%%%%%%%%%%%%%%%%%%%%%%%%%%%%%%%%%%%%%%%%%%%%%%%%%%%%%%%%%%%%
\section{\SGL regression}
%%%%%%%%%%%%%%%%%%%%%%%%%%%%%%%%%%%%%%%%%%%%%%%%%%%%%%%%%%%%%%%%%%%%%%%%%%%%%%%
%%%%%%%%%%%%%%%%%%%%%%%%%%%%%%%%%%%%%%%%%%%%%%%%%%%%%%%%%%%%%%%%%%%%%%%%%%%%%%%
We are interested in solving an estimation problem with penalty governed by $\Omega$, a sparsity inducing norm and a parameter $\lambda>0$ trading-off between data-fitting and sparsity. The primal problem reads:
\begin{equation}
 \label{eq:general_primal_problem} \!\!
 \tbeta {\lambda, \Omega} \in \argmin_{\beta \in \bbR^{p}}
\frac{1}{2} \norm{y - X\beta}^{2} + \lambda \Omega(\beta) : = P_{\lambda,\Omega}(\beta).
\end{equation}
A dual formulation (see \citet[Th.~3.3.5]{Borwein_Lewis06}) of
\eqref{eq:general_primal_problem} is given by
\begin{equation} \label{eq:general_dual_problem}
% \mathbf{Dual~problem}:
\ttheta {\lambda, \Omega} = \argmax_{\theta \in \dualomega}
\frac{1}{2} \norm{y}^{2} - \frac{\lambda^2}{2}
\norm{\theta - \frac{y}{\lambda}}^{2} : = D_{\lambda}(\theta),
\end{equation}
where $\dualomega = \{\theta \in \bbR^n: \Omega^D(X^{\top} \theta) \leq 1\}$.

Moreover, Fermat's rule reads:
\begin{align}
\lambda \ttheta {\lambda, \Omega} &= y - X \tbeta{\lambda, \Omega}
\quad\textbf{ (link-equation) }, \label{eq:link_equation}\\
\!\!\! {X}^{\top} \ttheta {\lambda, \Omega} &\in
\partial \Omega(\tbeta {\lambda,\Omega})
\textbf{ (sub-differential inclusion)}
\label{eq:sub-differential_inclusion}.
\end{align}

\begin{remark}[Dual uniqueness]
As for the Lasso problem, the dual solution $\ttheta{\lambda, \Omega}$
is unique, while the primal solution $\tbeta{\lambda, \Omega}$ might not be.
Indeed, the dual formulation \eqref{eq:general_dual_problem} is equivalent to
$\ttheta{\lambda, \Omega} = {\argmin}_{\theta \in \dualomega} \,
\norm{\theta - y/\lambda}$
and so $\ttheta{\lambda, \Omega} = \Pi_{\dualomega}(y/\lambda)$ is the
projection of $y/\lambda$ over the dual feasible (closed and convex) set $\dualomega$.
\end{remark}

\begin{remark}[Critical parameter: $\lambda_{\max}$]
There is a critical value $\lambda_{\max}$ such that $0$ is a primal solution of \eqref{eq:general_primal_problem} for all $\lambda \geq \lambda_{\max}$. Indeed, the Fermat's rule states:
\begin{align*}
0 &\in \argmin_{\beta \in \bbR^p} \frac{1}{2} \| y - X\beta\|^{2} +
\lambda \Omega(\beta) \\
&\overset{\eqref{th:Fermat_rule}}{\Longleftrightarrow}
0 \in \{X^\top y \} + \lambda \partial \Omega(0)
\overset{\eqref{eq:sub-differential_norm}}{\Longleftrightarrow}
\Omega^D (X^\top y) \leq \lambda.
\end{align*}
Hence, the critical parameter is given by:
\begin{equation}
 \label{eq:general_lambda_max}
 \lambda_{\max} := \Omega^D(X^\top y).
\end{equation}
\end{remark}

\begin{figure*}[!ht]
\subfigure[Lasso dual ball $\mathcal{B}_{\Omega^D}$ for  $\Omega^{D}(\theta)=\|\theta\|_\infty$.]
{\includegraphics[width=0.32\linewidth]{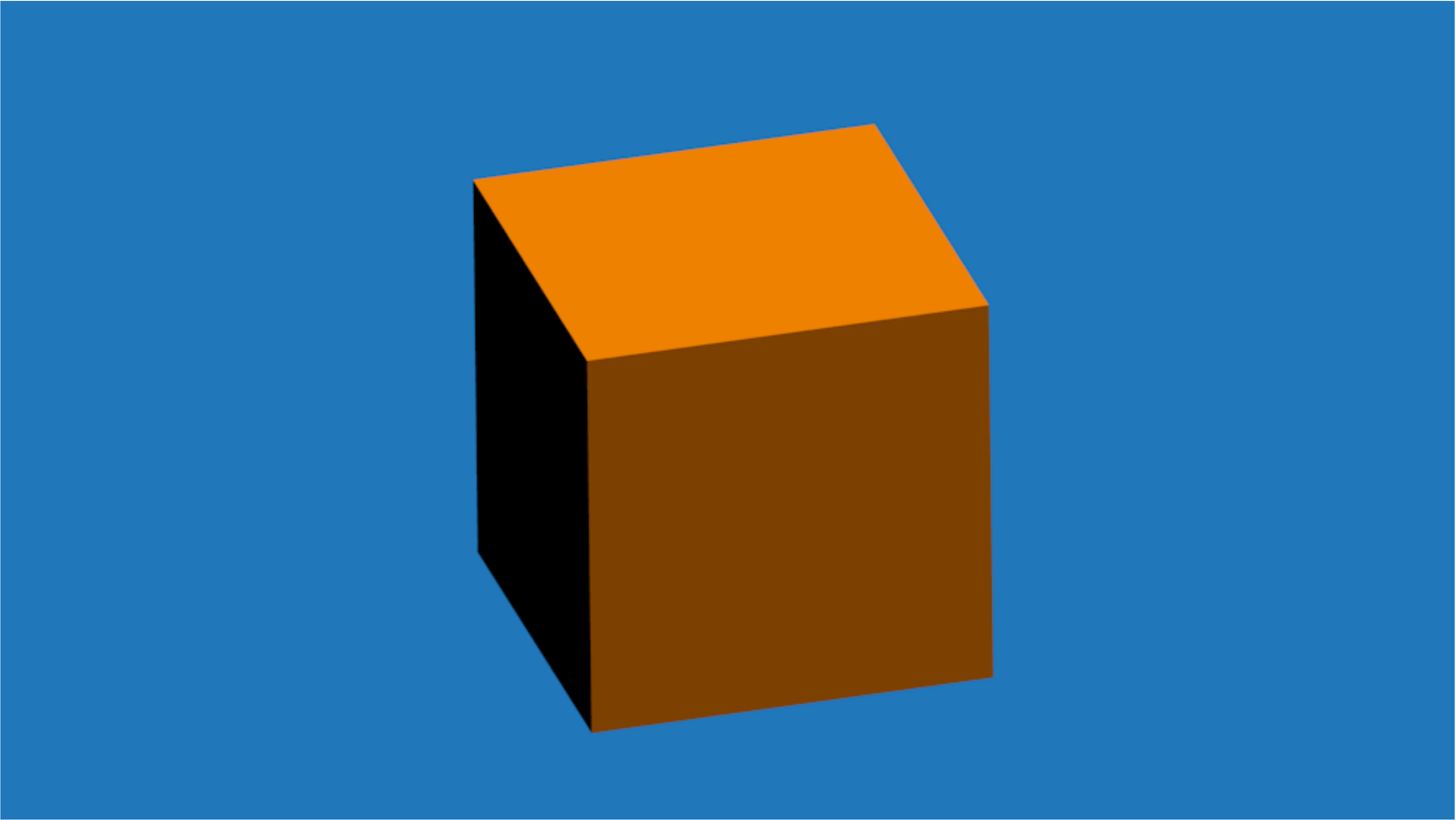}}
\hfill
\subfigure[Group-Lasso dual ball $\mathcal{B}_{\Omega^D}$ for $\Omega^{D}(\theta)=\max(\sqrt{\theta_1^2+\theta_2^2}, |\theta_3|)$.]
{\includegraphics[width=0.32\linewidth]{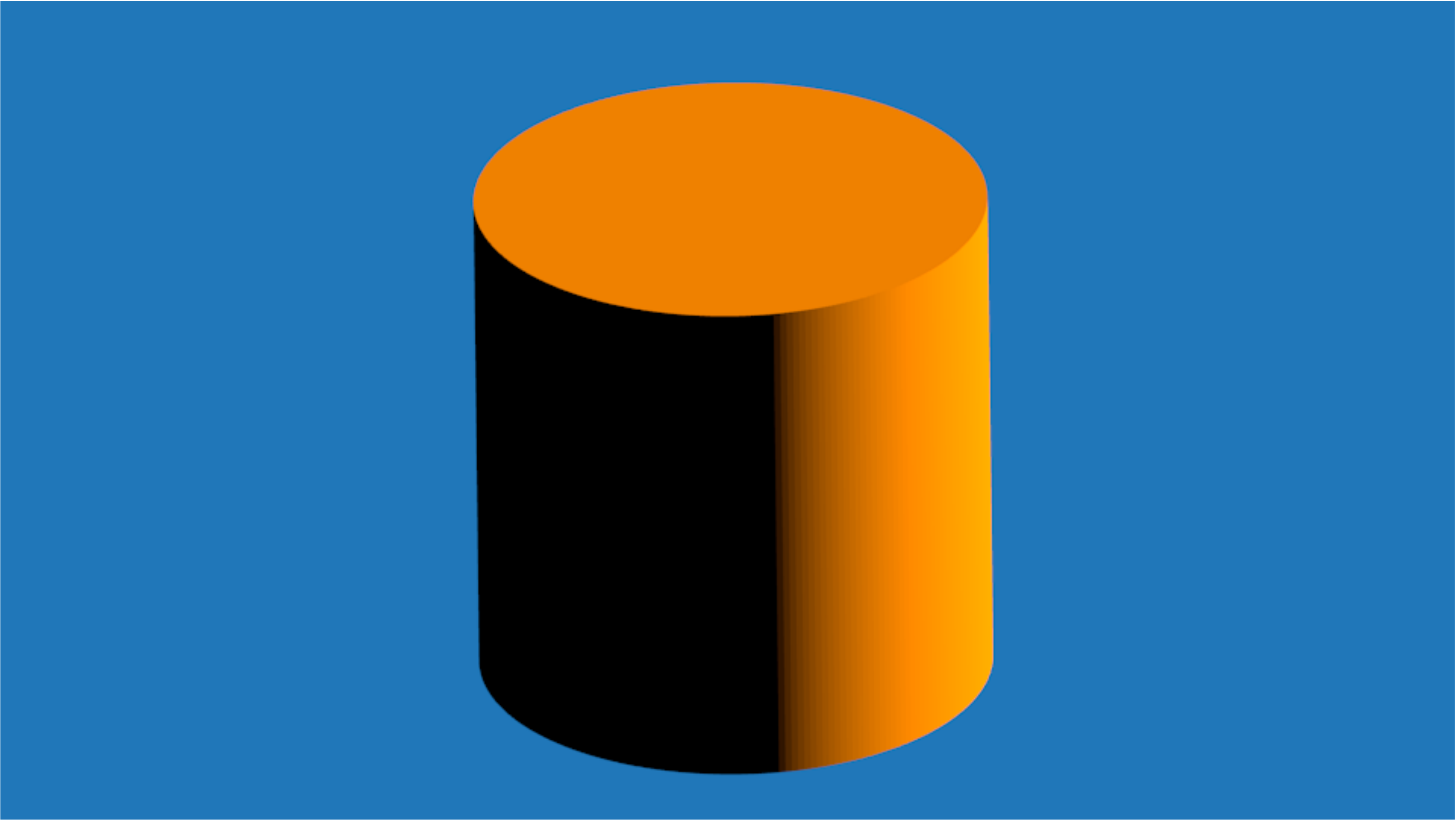}}
\hfill
\subfigure[\SGL dual ball $\mathcal{B}_{\Omega^D}=\big\{\theta: \forall g \in \mathcal{G},
\|\ST{\tau}(\theta_g)\| \leq (1 - \tau)w_g \big\}$.]{\includegraphics[width=0.32\linewidth]{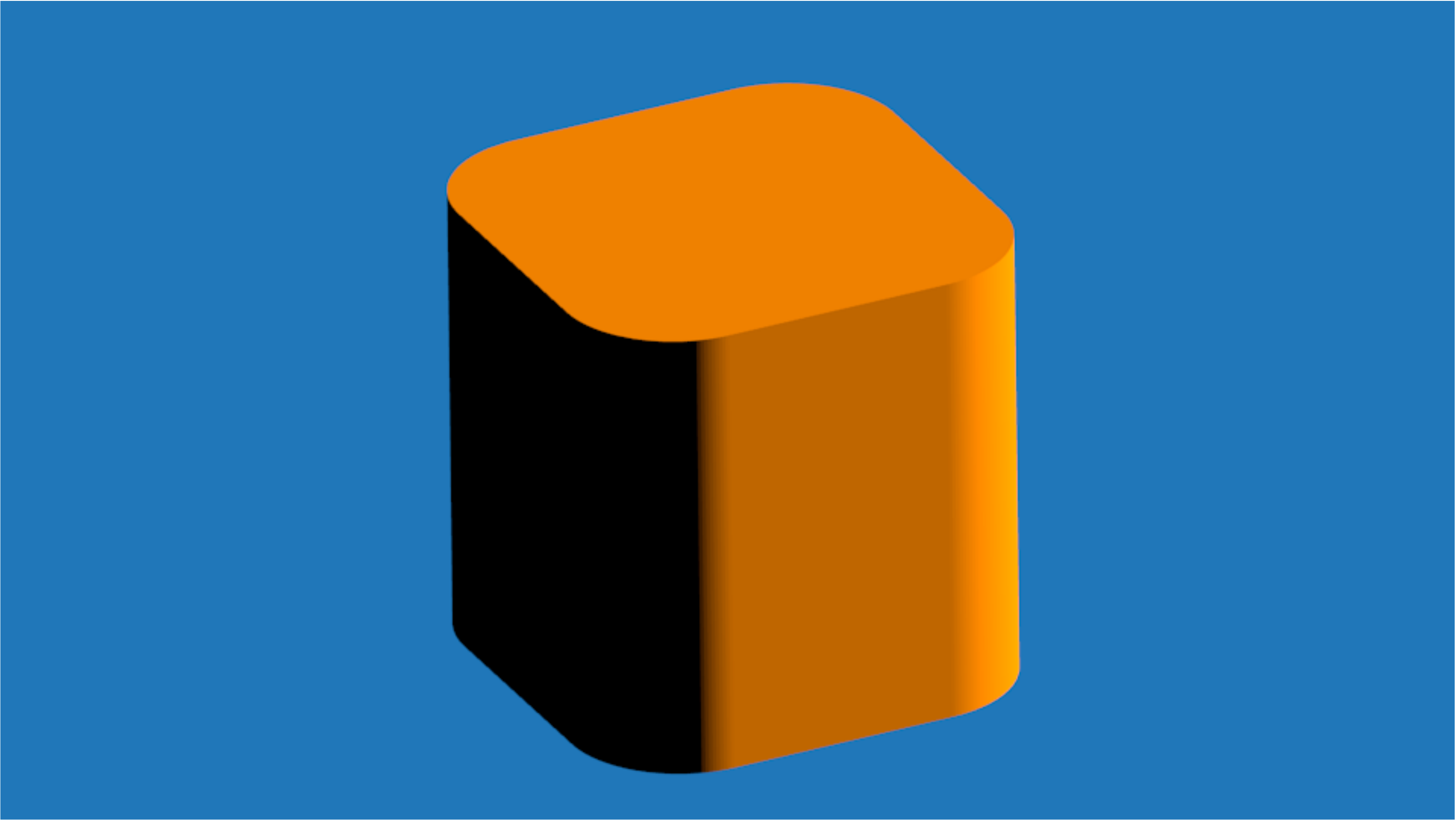}}
\caption{Lasso, Group-Lasso and \SGL dual unit balls $\mathcal{B}_{\Omega^D}=\{\theta : \Omega^{D}(\theta) \leq 1 \}$, for the case of $\mathcal{G}=\{ \{1,2\},\{3\} \}$ (\ie $g_1=\{1,2\},g_2=\{3\}$), $n=p=3$, $w_{g_1}=w_{g_2}=1$ and $\tau=1/2$.}
\label{fig:dual_balls}
\end{figure*}
In what follows, we are only interested in the \SGL norm $\Omega=\Omega_{\tau,w}$ defined by
\begin{equation} \label{eq:sgl_norm}
\normsgl(\beta) := \tau \|\beta\|_{1} + (1 - \tau) \sum_{g\in \mathcal{G}} w_g \norm{\beta_{g}},
\end{equation}
for $\tau \in [0,1], w = (w_g)_{g\in \mathcal{G}}$ with $w_g \geq 0$ for all $g\in \mathcal{G}$. The case where $w_g=0$ for some $g \in \mathcal{G}$ together with $\tau=0$ is excluded ($\normsgl$ is not a norm in such a case).

For $\lambda>0$ and $\tau \in [0,1]$, the \SGL estimator denoted by
$\tbeta{\lambda,\tau, w}$ is defined as a minimizer of the primal objective $P_{\lambda,\tau,w}:= P_{\lambda,\Omega_{\tau,w}}$ defined by \eqref{eq:general_primal_problem}, with the norm $\normsgl$. Similarly $\ttheta{\lambda, \tau, w}$ stands for the maximizer of the dual objective $D_{\lambda}$ over $\dualomegasgl$ in \eqref{eq:general_dual_problem}.

%\jo{FROM HERE NEED TO DOUBLE CHECK NOTATION $\dualomegasgl$ vs $\dualomega$, especially in the appendix} Done
\begin{remark} We recover the Lasso \cite{Tibshirani96} if $\tau=1$, and the group-Lasso \cite{Yuan_Lin06} if $\tau = 0$.
\end{remark}

%%%%%%%%%%%%%%%%%%%%%%%%%%%%%%%%%%%%%%%%%%%%%%%%%%%%%%%%%%%%%%%%%%%%%%%%%%%%%%%
%%%%%%%%%%%%%%%%%%%%%%%%%%%%%%%%%%%%%%%%%%%%%%%%%%%%%%%%%%%%%%%%%%%%%%%%%%%%%%%
\section{GAP safe rule for the \SGL}
%%%%%%%%%%%%%%%%%%%%%%%%%%%%%%%%%%%%%%%%%%%%%%%%%%%%%%%%%%%%%%%%%%%%%%%%%%%%%%%
%%%%%%%%%%%%%%%%%%%%%%%%%%%%%%%%%%%%%%%%%%%%%%%%%%%%%%%%%%%%%%%%%%%%%%%%%%%%%%%

The safe rule we propose here is an extension to the \SGL of the GAP safe rules introduced for the Lasso and the Group-Lasso \citep{Fercoq_Gramfort_Salmon15,Ndiaye_Fercoq_Gramfort_Salmon15}. For the \SGL, the geometry of the dual feasible set $\dualomegasgl$ is more complex (see Figure~\ref{fig:dual_balls}). As a consequence, additional geometrical insights are needed to derive efficient safe rules.

%%%%%%%%%%%%%%%%%%%%%%%%%%%%%%%%%%%%%%%%%%%%%%%%%%%%%%%%%%%%%%%%%%%%%%%%%%%%%%%
\subsection{Description of the screening rules}
%%%%%%%%%%%%%%%%%%%%%%%%%%%%%%%%%%%%%%%%%%%%%%%%%%%%%%%%%%%%%%%%%%%%%%%%%%%%%%%

Safe screening rules exploit the known sparsity of the solutions of problems such as \eqref{eq:general_primal_problem}. They discard inactive features whose coefficients are guaranteed to be zero for optimal solutions. Ignoring ``irrelevant'' features in the optimization can significantly reduce computation time.

The \SGL beneficiates from two levels of screening: the safe rules can detect both group-wise zeros in the vector $\tbeta {\lambda, \tau, w}$ and coordinate-wise zeros in the remaining groups. We now derive such properties.
\begin{proposition}[Theoretical screening rules]
\label{prop:theoretical_screening_rules}
The two levels of screening rules for the \SGL are: \\
\textbf{Feature level screening: }
\begin{equation*}
\forall j \in g, \,
|X_j^\top \ttheta{\lambda,\tau,w}| < \tau
\Longrightarrow \tbeta{\lambda,\tau,w}_j = 0.
\end{equation*}
\textbf{Group level screening: }
\begin{equation*}
\forall g \in \mathcal{G}, \,
\|\mathcal{S}_\tau(X_g^\top \ttheta{\lambda,\tau,w})\|
< (1-\tau) w_g \Longrightarrow \tbeta{\lambda,\tau,w}_g = 0.
\end{equation*}
\end{proposition}
\begin{proof}
The proof is given in the Appendix; see also \cite{Wang_Ye14}.
\end{proof}
\begin{remark}
The first rule is with a strict inequality, but it can be relaxed to a
non-strict inequality when $\tau \neq 1$.
\end{remark}

Note that the screening rules above are theoretical as stated since $\ttheta{\lambda,\tau,w}$ is inherently unknown. To get useful screening rules one needs a \textbf{safe region}, \ie a set that contains the optimal dual solution $\ttheta{\lambda,\tau,w}$.  When choosing a ball $\mathcal{B}(\theta_c,r)$ with radius $r$ and centered at $\theta_c$ as a safe region, we call it a safe sphere, following \citet{ElGhaoui_Viallon_Rabbani12}. A safe ball is all the more useful that $r$ is small and $\theta_c$ close to $\ttheta{\lambda,\tau,w}$. The safe rules for the \SGL reads: for any group $g$ in $\mathcal{G}$ and any safe ball $\mathcal{B}(\theta_c,r)$

\textbf{Group level safe screening rule:}
\begin{equation}
\label{eq:safe_rule_group}
 \max_{\theta \in \mathcal{B}(\theta_c,r)} \|\mathcal{S}_\tau (X_g^\top
\theta)\| <
 (1 - \tau) w_g \Rightarrow \tbeta{\lambda,\tau,w}_g  = 0.
\end{equation}
\textbf{Feature level safe screening rule: }
\begin{equation}
\label{eq:safe_rule_feature} \forall j \in g,
\max_{\theta \in \mathcal{B}(\theta_c,r)} |X_j^\top \theta| < \tau
\Rightarrow \tbeta{\lambda, \tau, w}_j=0.
\end{equation}
For screening variables, we rely on the upper-bounds on $\max_{\theta \in \mathcal{B}(\theta_c,
r)} |X_j^\top \theta|$ and  $\max_{\theta \in \mathcal{B}(\theta_c,r)}
\|\mathcal{S}_\tau (X_g^\top \theta)\|$ presented below (see also \citep{Wang_Ye14}). A new and shorter proof is given in the Appendix.

\begin{proposition}\label{prop:screening_bound}
For all group $g \in \mathcal{G}$ and $j \in g$,
\begin{equation}\label{eq:max_st_features}
\max_{\theta \in \mathcal{B}(\theta_c, r)} |X_j^\top \theta| \leq
|X_j^\top \theta_c| + r \|X_j\|.
\end{equation}
$\max_{\theta \in \mathcal{B}(\theta_c, r)} \| \ST{\tau} (X_g^\top \theta) \|$
is upper bounded by
\begin{equation}\label{eq:max_ST}
\begin{cases}
 \|\ST{\tau} (X_g^\top \theta_c)\| + r \|X_g\| &\text{ if }
 \|X_g^\top \theta_c\|_{\infty} > \tau, \\
 (\|X_g^\top \theta_c\|_{\infty} + r \|X_g\| - \tau)_+ &\text{ otherwise}.
\end{cases}
\end{equation}
\end{proposition}

\begin{remark}
Note that other kinds of safe regions can be use, for instance domes \citet{ElGhaoui_Viallon_Rabbani12}, but we only focus on safe sphere for simplicity. The experiments in \citep{Fercoq_Gramfort_Salmon15} have shown limited speed-ups when substituting domes to spheres (with same diameters).
\end{remark}

Assume one has found a safe sphere $\mathcal{B}(\theta_c, r)$, the safe rules given by \eqref{eq:safe_rule_group} and \eqref{eq:safe_rule_feature} read:
\begin{theorem}[Safe rules for the \SGL]
\textbf{Group level safe screening: }
\begin{equation*}
 \forall g \in \mathcal{G}, \text{ if } \mathcal{T}_{g} < (1 - \tau) w_g ,
\text{ then } \tbeta{\lambda, \tau, w}_g = 0, \text{ where }
\end{equation*}
\begin{equation*}
\mathcal{T}_{g} :=
\begin{cases}
 \|\ST{\tau} (X_g^\top \theta_c)\| + r \|X_g\| &\text{ if }
 \|X_g^\top \theta_c\|_{\infty} > \tau, \\
 (\|X_g^\top \theta_c\|_{\infty} + r \|X_g\| - \tau)_+ &\text{ otherwise }.
\end{cases}
\end{equation*}
\textbf{Feature level safe screening: }
 \begin{equation*}
 \forall g \in \mathcal{G},
  \forall j \in g: \text{ if }
  |X_j^\top \theta_c| + r \|X_j\| < \tau,
  \text{ then } \tbeta{\lambda, \tau, w}_j = 0.
 \end{equation*}
\end{theorem}
%\jo{where are such rules proved to be safe/ok?}
\begin{proof}
Combining \eqref{eq:safe_rule_group} with \eqref{eq:max_ST} yields the group level safe screening. Combining \eqref{eq:safe_rule_feature} with \eqref{eq:max_st_features} yields the feature level safe screening.
\end{proof}

The screening rules above show us which coordinates or group of
coordinates can be safely set to zero. As a consequence, we can remove the corresponding features from the design matrix $X$ during the optimization process. While standard algorithms solve the problem \eqref{eq:general_primal_problem} scanning all variables, only active ones \ie non screened-out variables (\lcf Section~\ref{subesec:active_set} for details) need to be considered with safe screening strategies. This leads to significant computational speed-ups, especially with a coordinate descent algorithm for which it is natural to ignore features (see Algorithm \ref{alg:ista_bc_safe}). Now, let us show how to compute efficiently the radius $r$ and the dual feasible point $\theta$ for the \SGL, using the duality gap.

%%%%%%%%%%%%%%%%%%%%%%%%%%%%%%%%%%%%%%%%%%%%%%%%%%%%%%%%%%%%%%%%%%%%%%%%%%%%%%%
\subsection{GAP Safe sphere}
%%%%%%%%%%%%%%%%%%%%%%%%%%%%%%%%%%%%%%%%%%%%%%%%%%%%%%%%%%%%%%%%%%%%%%%%%%%%%%%

%%%%%%%%%%%%%%%%%%%%%%%%%%%%%%%%%%%%%%%%%%%%%%%%%%%%%%%%%%%%%%%%%%%%%%%%%%%%%%%
\subsubsection{Computation of the radius}
%%%%%%%%%%%%%%%%%%%%%%%%%%%%%%%%%%%%%%%%%%%%%%%%%%%%%%%%%%%%%%%%%%%%%%%%%%%%%%%

With a dual feasible point $\theta \in \dualomegasgl$ and a primal
vector $\beta\in \bbR^{p}$ at hand, let us construct a safe
sphere centered on $\theta$, with radius %$r_{\lambda} (\beta, \theta)$
obtained thanks to dual gap computations.

\begin{theorem}[Safe radius]\label{th:GAP_Safe_sphere}
For any $\theta \in \dualomegasgl$ and any $\beta\in \bbR^p$, one has
$ \ttheta{\lambda,\tau, w} \in
\mathcal{B}\left(\theta,\bestR[\lambda,\tau]{\beta}{\theta}\right),$ for
\begin{align*}
% \text{ with }
\bestR[\lambda,\tau]{\beta}{\theta}&=\sqrt{\frac{2(P_{\lambda,\tau,w}(\beta)-
D_{\lambda}(\theta))}{\lambda^2}},
\end{align*}
\ie the aforementioned ball is a safe
region for the \SGL problem.
\end{theorem}
\begin{proof}
This results holds thanks to strong concavity of the dual objective. A complete proof is given in the Appendix.
\end{proof}

%%%%%%%%%%%%%%%%%%%%%%%%%%%%%%%%%%%%%%%%%%%%%%%%%%%%%%%%%%%%%%%%%%%%%%%%%%%%%%%
\subsubsection{Computation of the center}
%%%%%%%%%%%%%%%%%%%%%%%%%%%%%%%%%%%%%%%%%%%%%%%%%%%%%%%%%%%%%%%%%%%%%%%%%%%%%%%

In GAP safe screening rules, the screening test relies crucially on the ability to compute a vector that belongs to the dual feasible set. Following \citet{Bonnefoy_Emiya_Ralaivola_Gribonval15}, we leverage the primal/dual link-equation \eqref{eq:link_equation} to dynamically construct a dual point based on a current approximation $\beta_k$ of $\tbeta{\lambda, \tau, w}$. Note that here $\beta_k$ is the primal value at iteration $k$ obtained by an iterative algorithm. Starting from a current residual $\rho_k = y - X\beta_k$, one can create a dual feasible point by\footnote{We have used a simpler scaling w.r.t. \citet{Bonnefoy_Emiya_Ralaivola_Gribonval14} choice's (without noticing much difference): $\theta_k = s \rho_k$ where
$s = \min \left[
 \max \left(\frac{\rho_k^\top y }{\lambda \norm{\rho_k}^{2}},
	     \frac{-1}{\normsgl^D(X^\top \rho_k)}\right),
 \frac{1}{\normsgl^D(X^\top \rho_k)} \right]$.}
 choosing for all $k \in \bbN$:
\begin{equation}\label{eq:dual_feasible_point}
\theta_k = \frac{\rho_k}{\max(\lambda, \normsgl^D(X^\top \rho_k))}.
\end{equation}
We refer to $\mathcal{B}(\theta_k, \bestR[\lambda,\tau]{\beta_k}{\theta_k})$ as GAP safe spheres.

\begin{remark}
Recall that $\lambda \geq \lambda_{\max}$ yields $\tbeta {\lambda, \tau, w} = 0$, in which case $\rho := y - X \tbeta {\lambda, \tau, w} = y$ is the optimal residual and $y/\lambda_{\max}$ is the dual solution. Thus, as for getting $\lambda_{\max} = \normsgl^D(X^\top y)$, the scaling computation in \eqref{eq:dual_feasible_point} requires a dual norm evaluation.
\end{remark}

\subsection{Convergence of the active set}\label{subesec:active_set}
Let us recall the notion of converging safe regions introduced in \cite{Fercoq_Gramfort_Salmon15}.
\begin{definition}
Let $(\mathcal{R}_{k})_{k \in \bbN}$ be a sequence of closed convex sets in
$\bbR^n$ containing $\ttheta{\lambda, \tau, w}$. It is a converging sequence of
safe regions if the diameters of the sets converge to zero.
\end{definition}

The following proposition states that the sequence of dual feasible points obtained from \eqref{eq:dual_feasible_point} converges to the dual solution $\ttheta{\lambda,\tau,w}$ if $(\beta_k)_{k \in
\bbN}$ converges to an optimal primal solution $\tbeta{\lambda,\tau,w}$ (the proof is in the Appendix).
\begin{proposition} \label{prop:dual_convergence}
 If $\lim_{k \to \infty} \beta_k = \tbeta{\lambda, \tau, w} $, then
 $\lim_{k \to \infty} \theta_k = \ttheta{\lambda, \tau, w} $.
\end{proposition}
\begin{remark}
This proposition guarantees that the GAP safe spheres $\mathcal{B}(\theta_k, \bestR[\lambda,\tau]{\beta_k}{\theta_k})$ are converging safe regions in the sense introduced by \citet{Fercoq_Gramfort_Salmon15}, since by strong duality $\lim_{k \rightarrow \infty}
\bestR[\lambda,\tau]{\beta_k}{\theta_k} = 0$.
\end{remark}

For any safe region $\mathcal{R}$, \ie containing $\ttheta{\lambda, \tau, w}$, we define two levels of active sets:
\begin{align*}
\mathcal{A}_{\text{groups}}(\mathcal{R}) &:=
\left\{ g \in \mathcal{G}, \, \max_{\theta \in \mathcal{R}}
\|\mathcal{S}_{\tau}(X_{g}^{\top} \theta)\| \geq (1 - \tau) w_g \right\}, \\
\mathcal{A}_{\text{features}}(\mathcal{R}) &:=
\bigcup_{g \in \mathcal{A}_{\text{groups}}(\mathcal{R})}
\left\{ j \in g: \, \max_{\theta \in \mathcal{R}}
|X_j^\top \theta| \geq \tau \right\}.
\end{align*}
If one considers sequence of converging regions, then the next proposition states that we can identify, in finite time, the optimal active sets defined as follows (see
Appendix):
\begin{align*}
\mathcal{E}_{\text{groups}} &:= \left\{ g \in \mathcal{G}: \,
\|\mathcal{S}_{\tau}(X_{g}^{\top} \ttheta{\lambda,\tau,w})\|
= (1 - \tau) w_g \right\}, \\
\mathcal{E}_{\text{features}} &:= \bigcup_{g \in
\mathcal{E}_{\text{groups}}}
\left\{ j \in g: \, |X_j^\top \ttheta{\lambda,\tau,w}| \geq \tau \right\}.
\end{align*}
\begin{proposition}
\label{prop:convergence_regions}
Let $(\mathcal{R}_k)_{k \in \bbN}$ be a sequence of safe regions whose diameters converge to 0. Then,
$\displaystyle \lim_{k \rightarrow \infty}
\mathcal{A}_{\text{groups}}(\mathcal{R}_k) = \mathcal{E}_{\text{groups}}$ and
$\displaystyle \lim_{k \rightarrow \infty}
\mathcal{A}_{\text{features}}(\mathcal{R}_k) = \mathcal{E}_{\text{features}}$.
\end{proposition}

%%%%%%%%%%%%%%%%%%%%%%%%%%%%%%%%%%%%%%%%%%%%%%%%%%%%%%%%%%%%%%%%%%%%%%%%%%%%%%%
%%%%%%%%%%%%%%%%%%%%%%%%%%%%%%%%%%%%%%%%%%%%%%%%%%%%%%%%%%%%%%%%%%%%%%%%%%%%%%%
\section{Properties of the \SGL}
\label{section:properties_of_sgl}
%%%%%%%%%%%%%%%%%%%%%%%%%%%%%%%%%%%%%%%%%%%%%%%%%%%%%%%%%%%%%%%%%%%%%%%%%%%%%%%
%%%%%%%%%%%%%%%%%%%%%%%%%%%%%%%%%%%%%%%%%%%%%%%%%%%%%%%%%%%%%%%%%%%%%%%%%%%%%%%
The remaining ingredient for creating our GAP safe screening rule is a way to perform the evaluation of the dual norm $\normsgl^D$, which we describe hereafter along with some useful properties of the norm $\normsgl$.
% Recall that the computation of $\lambda_{\max}$ and of dual feasible points require dual norm evaluations.
Such evaluations need to be performed multiple times during the algorithm. This motivates the derivation of the efficient Algorithm~\ref{alg:compute_lambda} presented in this section.

%%%%%%%%%%%%%%%%%%%%%%%%%%%%%%%%%%%%%%%%%%%%%%%%%%%%%%%%%%%%%%%%%%%%%%%%%%%%%%%
\subsection{Connections with $\epsilon$-norms}
%%%%%%%%%%%%%%%%%%%%%%%%%%%%%%%%%%%%%%%%%%%%%%%%%%%%%%%%%%%%%%%%%%%%%%%%%%%%%%%

Here, we establish a link between the \SGL norm $\normsgl$ and the
$\epsilon$-norm (denoted $\norm{\cdot}_{\epsilon}$) introduced in \cite{Burdakov88}.
% \alex{say early why it is interesting}
For any $\epsilon \in [0,1]$ and any $x\in \bbR^d$,
$\norm{x}_{\epsilon}$ is defined as the unique nonnegative solution $\nu$ of
the following equation:
\begin{equation}\label{eq:1_burdakov_norm}
 \sum_{i=1}^d(|x_i| - (1 - \epsilon)\nu)_+^2 = (\epsilon \nu)^2,
\end{equation}
Using soft-thresholding, this is equivalent to:
\begin{equation}
\label{eq:2_burdakov_norm}
\sum_{i = 1}^{d} {\ST{(1 - \epsilon)\nu}(x_i)}^2 =
  \norm{\ST{(1 - \epsilon)\nu}(x)}^{2} = (\epsilon \nu)^2.
\end{equation}
Moreover, the dual norm of the $\epsilon$-norm is defined by\footnote{see \citep[Eq. (42)]{Burdakov_Merkulov01} or Appendix}:
% \jo{do the proof from Burdakov for completeness if time+motivation}):
%
\begin{equation*}
\norm{y}_{\epsilon}^{D} = \epsilon \norm{y}^{D} +
(1 - \epsilon) \norm{y}_{\infty}^{D} =
\epsilon \norm{y} + (1 - \epsilon) \norm{y}_{1}.
\end{equation*}
Now we can express the \SGL norm $\normsgl$ in term of the $\epsilon$-dual-norm and derive some basic properties.

% \begin{remark}
% $\tau \neq 0$ or $w_g \neq 0$ (when $\tau=0$ there is no contribution of the
% $\|\cdot\|_1$ norm when $w_g=0$ the $g-$th group is not penalized).
% \end{remark}

\begin{proposition}\label{prop:properties_of_sgl}
For all groups $g$ in $\mathcal{G}$, let us introduce
\begin{equation}
\epsilon_g := \frac{(1 - \tau)w_g}{\tau + (1 - \tau)w_g}.
\end{equation}
Then, the \SGL norm satisfies the following properties:  for any $\beta$ and $\xi$ in $\bbR^p$
% \vspace{-0.2cm}
\begin{align}
&\normsgl(\beta) = \sum_{g \in \mathcal{G}} (\tau + (1 - \tau) w_g)
\|\beta_g\|_{\epsilon_g}^{D},
\label{eq:epsilon_norm} \\
% = \sum_{g \in \mathcal{G}} \frac{(1 - \tau) w_g}{\epsilon_g}
% \|\beta_g\|_{\epsilon_g}^{D}, \\
&\normsgl^{D}(\xi) = \max_{g \in \mathcal{G}}
\frac{\norm{\xi_g}_{\epsilon_g}}{\tau + (1 - \tau) w_g},
\label{eq:dual_norm_computation}\\
% \max_{g \in \mathcal{G}} \frac{(1- \tau) w_g}{\epsilon_g}
% \|\xi_g\|_{\epsilon_g},
%&\iota_{\mathcal{B}_{\normsgl^D}} (\xi) =
%\sum_{g \in\mathcal{G}} \iota_{(1 - \tau) w_g \mathcal{B}}
%\left(\xi_g - \Pi_{\tau \mathcal{B}_\infty}(\xi_g)\right), \\
&\mathcal{B}_{\normsgl^D} \!\!\!=
\big\{\xi \in \bbR^p : \forall g \in \mathcal{G},
\|\ST{\tau}(\xi_g)\| \leq (1 - \tau)w_g \big\}.
\label{eq:unit_dual_ball}
\end{align}
The sub-differential $\partial\normsgl(\beta)$ of the norm $\normsgl$ at $\beta$ is
\begin{equation*}
\bigg\{z \in \bbR^p: \forall g \in
\mathcal{G}, z_g \in \tau \partial \|\cdot\|_{1}(\beta_g) +
(1 - \tau) w_g \partial \|\cdot\|(\beta_g) \bigg\}
\end{equation*}
\end{proposition}
\begin{remark}[Decomposition of a dual feasible point]
We obtain from the sub-differential inclusion \eqref{eq:sub-differential_norm} and the characterization
of the unit dual ball \eqref{eq:unit_dual_ball} that for the \SGL any dual
feasible point $\theta \in \dualomegasgl$ verifies:
\begin{equation*}\label{rq:decomposition_of_dual_feasible_point}
\forall g \in \mathcal{G}, \quad X_{g}^{\top} \theta \in (1 - \tau) w_g
\mathcal{B} + \tau \mathcal{B}_{\infty}.
\end{equation*}
\end{remark}
From the dual norm formulation \eqref{eq:dual_norm_computation}, a vector $\theta \in \bbR^n$ is feasible if and only if $\normsgl^D(X^\top
\theta) \leq 1$, \ie $\forall g \in \mathcal{G}, \|X_{g}^\top \theta\|_{\epsilon_g} \leq \tau + (1 - \tau) w_g$. Hence we deduce from \eqref{eq:unit_dual_ball} a new characterization of the dual feasible set:

\begin{proposition}[Dual feasible set and $\epsilon$-norm] \label{rq:dual_feasible_set_burdakov_characterization}
\begin{align*}
\dualomegasgl &=
% \big\{\theta \in \bbR^n : \forall g \in \mathcal{G},
% \|\ST{\tau}(X_{g}^\top \theta)\| \leq (1 - \tau)w_g \big\} \\
%  &=
 \big\{\theta \in \mathbb{R}^n: \forall g \in \mathcal{G},
\|X_{g}^\top \theta\|_{\epsilon_g} \leq \tau + (1 - \tau)w_g \big\}.
\end{align*}
\end{proposition}

%\jo{Need to fix the case alpha=R=0!!! in the algorithm}
\renewcommand{\algorithmicloop}{
\textbf{First method}}
\begin{algorithm}[!Ht]
\caption{Computation of $\Lambda(x,\alpha,R)$.}
\begin{algorithmic} \label{alg:sgl_burdakov}
\INPUT{$x=(x_1,\ldots,x_d)^\top \in \bbR^d, \quad \alpha \in [0, 1], \quad R \geq 0$}
\IF{$\alpha = 0$ and $R = 0$}
    \STATE $\Lambda(x,\alpha,R) = \infty$
\ELSIF{$\alpha = 0$ and $R \neq 0$}
   \STATE $\Lambda(x,\alpha,R) = \norm{x} / R$
\ELSIF{$R = 0$}
  \STATE $\Lambda(x,\alpha,R) = \norm{x}_{\infty} / \alpha$
\ELSE
  % \STATE $\text{Get } I := \left\{i \in [d]:\!|x_i|\!>\!\frac{\alpha \norm{x}_{\infty}}{\alpha + R}\right\} \text{ and } n_{I} := \text{Card}(I)$
  \STATE $\text{Get }  n_{I} :=  \text{Card} \left( \left\{i \in [d]:\!|x_i|\!>\!\frac{\alpha \norm{x}_{\infty}}{\alpha + R}\right\} \right)$
%   \vspace*{-0.5cm}
  \STATE Sort $x_{(1)} \geq x_{(2)} \geq \cdots \geq x_{(n_{I})}$
%   \vspace*{-0.4cm}
  \STATE $S_0 = x_{(0)}, \quad S_{0}^{(2)} = x_{(0)}^{2}, \quad a_0 = 0$
  \FOR{$ k \in [n_{I}- 1]$}
    \STATE $S_k = S_{k-1} + x_{(k)}; \quad
            S_{k}^{(2)} = S_{k-1}^{(2)} + x_{(k)}^{2}$
    \STATE $a_{k+1} = \frac{S_{k}^{(2)}}{x_{(k+1)}^{2}} -
            2 \frac{S_{k}}{ x_{(k+1)} } + k + 1$
    \IF{$ \frac{R^2}{\alpha^2} \in \left[a_k, a_{k+1}\right[$}
      \STATE $j_0 = k + 1$
      \STATE {\bf break}
    \ENDIF
  \ENDFOR
  \IF {$\alpha^2 j_0 - R^2 = 0$}
      \STATE $\Lambda(x,\alpha,R) = \frac{S_{j_0}^{2}}{2 \alpha S_{j_0}}$
  \ELSE
      \STATE $\Lambda(x,\alpha,R) = \frac{
\alpha S_{j_0} - \sqrt{\alpha^2 S_{j_0}^{2} - S_{j_0}^{(2)}(\alpha^2 j_0 - R^2)}
}{\alpha^2 j_0 - R^2}$ %(see \eqref{eq:solve_lambda}).
  \ENDIF
\ENDIF
\OUTPUT{$\Lambda(x,\alpha,R)$}
\end{algorithmic}
\label{alg:compute_lambda}
\end{algorithm}

\subsection{Efficient computation of the dual norm}
% Reminding Eq.~\eqref{eq:dual_norm_computation}, where $\norm{\xi_g}_{\epsilon_g}$ is the unique positive solution $\nu$ of the equation $\sum_{i = 1}^{n_g} {\ST{(1 - \epsilon_g)\nu}([\xi_g]_i)}^2 = (\epsilon_g \nu)^2$.
The following proposition shows how to compute the dual norm of the \SGL (and the $\epsilon$-norm), a crucial tool for our safe rules. This is turned into an efficient procedure in Algorithm \ref{alg:sgl_burdakov} (see the Appendix for more details).

\begin{proposition}\label{th:sparse_groupe_burdakov}
For $\alpha \in [0,1], R\geq0$ and $x\in\bbR^d$, the equation $\sum_{i=1}^d
\ST{\nu \alpha}(x_i)^2 = (\nu R)^2$ has a unique
solution $\nu \in \bbR_+$, denoted by $\Lambda(x,\alpha,R)$ and that can be
computed in $O(d \log d)$ operations in the worst case.
\end{proposition}

\begin{remark}
The complexity of Algorithm \ref{alg:compute_lambda} is $n_I \log(n_I)$ where $n_I = \text{Card} \left\{i \in [d]:|x_i|>\alpha \norm{x}_{\infty}/(\alpha + R)\right\}$ is often much smaller than the ambient dimension $d$.
\end{remark}

\begin{remark}
Thanks to \eqref{eq:general_lambda_max}, we can easily deduce the
critical parameter $\lambda_{\max}$ for the \SGL that is
\begin{equation}\label{eq:computation_alg_lambda_max}
\lambda_{\max} =
\max_{g\in \mathcal{G} } \frac{\Lambda(X_g^\top y, 1 - \epsilon_g, \epsilon_g)}{\tau + (1-\tau) w_g} = \normsgl^D(X^\top y),
\end{equation}
and compute a dual feasible point \eqref{eq:dual_feasible_point}, since
\begin{equation}\label{eq:dual_scaling}
\normsgl^D(X^\top \rho_k) = \max_{g\in \mathcal{G} } \frac{\Lambda(X_g^\top \rho_k,
1 - \epsilon_g, \epsilon_g)}{\tau + (1-\tau) w_g}.
\end{equation}
\end{remark}

%%%%%%%%%%%%%%%%%%%%%%%%%%%%%%%%%%%%%%%%%%%%%%%%%%%%%%%%%%%%%%%%%%%%%%%%%%%%%%%
%%%%%%%%%%%%%%%%%%%%%%%%%%%%%%%%%%%%%%%%%%%%%%%%%%%%%%%%%%%%%%%%%%%%%%%%%%%%%%%
\section{Implementation}
%%%%%%%%%%%%%%%%%%%%%%%%%%%%%%%%%%%%%%%%%%%%%%%%%%%%%%%%%%%%%%%%%%%%%%%%%%%%%%%
%%%%%%%%%%%%%%%%%%%%%%%%%%%%%%%%%%%%%%%%%%%%%%%%%%%%%%%%%%%%%%%%%%%%%%%%%%%%%%%

\begin{algorithm}[!Ht]
\caption{ISTA-BC with GAP SAFE rules}
\begin{algorithmic}\label{alg:ista_bc_safe}
\INPUT{$X, y , \epsilon, K, f^{\rm{ce}}, (\lambda_t)_{t \in [T-1]}$}
% \alex{what is $f$ here? $f$ is the primal cost elsewhere}
\STATE{\quad $\forall g \in \mathcal{G}$, compute $L_g=\|X_g\|_2^2 $}
\STATE \quad Compute $\lambda_0=\lambda_{\max}$ thanks to \eqref{eq:computation_alg_lambda_max} and Algorithm
\ref{alg:sgl_burdakov}
\STATE \quad $\beta^{\lambda_0}=0$
\FOR{$ t \in [T-1]$}
\STATE $\forall g \in \mathcal{G}, \alpha_g \leftarrow \lambda_t/L_g$
\STATE $\beta \leftarrow \beta^{\lambda_{t-1}}$ \COMMENT{previous $\epsilon$-solution}
% ($\beta^{\lambda_t} = 0$ if $t=1$)
\FOR{$k \in [K]$}
% \LOOP
\IF{$ k\mod f^{\rm{ce}} = 1 $}
\STATE Compute $\theta$ thanks to \eqref{eq:dual_feasible_point} and Algorithm
\ref{alg:sgl_burdakov}.
\STATE Set $\mathcal{R} = \mathcal{B} \left(\theta,
\sqrt{\frac{2 (P_{\lambda_t,\tau,w}(\beta) - D_{\lambda_t}(\theta) )}{\lambda_t^2}} \right)$
\IF{$P_{\lambda_t,\tau,w}(\beta) - D_{\lambda_t}(\theta) \leq \epsilon$}
\STATE $\beta^{\lambda_t}\leftarrow \beta$
\STATE {\bf break}
\ENDIF
\ENDIF
\FOR[Active groups]{$g \in \mathcal{A}_{\text{groups}}(\mathcal{R})$ }
% \STATE $\alpha_g = \frac{\lambda}{L_g}$
  \FOR[Active features]{$j \in g
\cap \mathcal{A}_{\text{features}}(\mathcal{R})$}
\STATE $\beta_j \leftarrow \mathcal{S}_{ \tau \alpha_g } \left( \beta_{j} -
     \frac{\nabla_j f(\beta)}{L_g} \right)$ \COMMENT{Soft-thresholding}
  \ENDFOR
\STATE $\beta_{g} \leftarrow \GST{(1-\tau)\omega_g \alpha_g} \left(\beta_g\right)$ \COMMENT{Block Soft-thresholding}
\ENDFOR
\ENDFOR
\ENDFOR
\OUTPUT{$(\beta^{\lambda_t})_{t\in [T-1]}$}
\end{algorithmic}
\label{alg:cd_screening}
\end{algorithm}

In this Section we provide details on how to solve the \SGL primal problem,
and how we apply the GAP safe screening rules. We focus on the \textit{block coordinate iterative soft-thresholding algorithm (ISTA-BC)}; see \citep{Qin_Scheinberg_Goldfarb13}.

This algorithm requires a block-wise Lipschitz gradient condition on the data fitting term $f(\beta) = \frac{1}{2}\norm{y - X \beta}^2$.
For our problem \eqref{eq:general_primal_problem}, one can show that for all group $g$ in $\mathcal{G}, L_g = \norm{X_g}_{2}^{2}$ (where $\|\cdot\|_2$ is the spectral norm of a matrix) is a suitable block-wise Lipschitz constant.
%, since it is the block-wise Lipschitz constant of the smooth part in the primal program \eqref{eq:general_primal_problem}.
%\jo{Olivie can you double check here, not really clear now,$f$ is needed, so is $L_g$ etc...}
We thus have a quadratic bound available on the variation of $f$
along each block, using \citep[Lemma 1.2.3]{Nesterov04}.

We define the block coordinate descent algorithm according to the
 Majorization-Minimization principle:
at each iteration $l$, we choose a group $g$ and the next iterate $\beta^{l+1}$ is defined such that $\beta^{l+1}_{g'} = \beta^{l}_{g'}$ if $g' \neq g$ and
otherwise
\begin{align*} \label{eq:sol_ISTA_BC_problem}
\beta_{g}^{l+1} \! &= \! \argmin_{\beta_g \in \bbR^{n_g}} \!
                 \frac{1}{2} \norm{\beta_g - \left( \!\! \beta_{g}^{l} -
\frac{\nabla_g f(\beta^{l})}{L_g}\! \right)}^{2} \!\!\!\! \\
& \hspace{4em} + \!\frac{\lambda}{L_g} \big( \tau\norm{\beta_g}_1 + (1-\tau)w_g \norm{\beta_g}\big) \\
    &= \GST{(1-\tau)\omega_g \alpha_g}
   \left( \mathcal{S}_{ \tau \alpha_g } \left( \beta_{g}^{l} -
\frac{\nabla_g f(\beta^{l})}{L_g} \right) \right),
\end{align*}
where we denote for all $g$ in $\mathcal{G}, \alpha_g := \frac{\lambda}{L_g}$.
In our implementation, we chose the groups in a cyclic fashion over the set of active groups.

The expensive computation of the dual gap is not performed at each pass over the data, but only every $f^{\rm{ce}}$ pass (in practice $f^{\rm{ce}}=10$ in all our experiments).

%!TEX root = ../icml.tex

\section{Experiments}

\subsection{Numerical experiments}
In our experiments\footnote{The source code can be found in
 \url{https://github.com/EugeneNdiaye/GAPSAFE_SGL}.}, we run Algorithm
  \ref{alg:ista_bc_safe} to obtain the \SGL estimator with a non-increasing sequence of $T$ regularization parameters $(\lambda_t)_{t \in [T-1]}$ defined as follows: $\lambda_t: = \lambda_{\max} 10^{- \frac{\delta t}{T - 1}}$. By default, we choose $\delta = 3$ and $T=100$, following the standard practice when running cross-validation using sparse models (see R GLMNET package \cite{Friedman_Hastie_Hofling_Tibshirani07}). The weights are always chosen as $w_g=\sqrt{n_g}$ (as in \cite{Simon_Friedman_Hastie_Tibshirani13}).

We also provide a natural extension of the previous safe rules \cite{ElGhaoui_Viallon_Rabbani12, Xiang_Xu_Ramadge11,Bonnefoy_Emiya_Ralaivola_Gribonval14} to the \SGL for comparisons (please refer to the appendix for more details). The \textbf{static safe region} \citep{ElGhaoui_Viallon_Rabbani12} is given by $\mathcal{B} \left(y/\lambda, \norm{y/\lambda_{\max} - y/\lambda} \right)$. The corresponding \textbf{dynamic safe region} \citep{Bonnefoy_Emiya_Ralaivola_Gribonval14}) is given by $\mathcal{B} \left(y/\lambda, \norm{\theta_k - y/\lambda} \right)$ where $(\theta_k)_{k \in \bbN}$ is a sequence of dual feasible points obtained by dual scaling; \lcf Equation~\eqref{eq:dual_feasible_point}. The \textbf{DST3}, which is an improvement of the preceding safe region (see also \citet{Xiang_Xu_Ramadge11, Bonnefoy_Emiya_Ralaivola_Gribonval14}), is the sphere $\mathcal{B}(\theta_c, r_{\theta_k})$ where
\begin{align*}
\theta_c &:= \frac{y}{\lambda} -
\frac{\frac{\eta^\top y }{\lambda} -
(\tau + (1 - \tau) w_{g_\star})}{\norm{\eta}^{2}} \eta, \\
r_{\theta_k}^2 &:= {\norm{\frac{y}{\lambda} - \theta_k}^{2} -
\norm{\frac{y}{\lambda} - \theta_c}^{2}},
\end{align*}
\begin{align*}
g_{\star} &:= \argmax_{g \in \mathcal{G}} \normsgl^D(X_{g}^{\top}y), \,
\epsilon_{g_{\star}} := \frac{(1 - \tau)w_{g_{\star}}}{\tau + (1 - \tau)w_{g_{\star}}}, \\
\eta &:=  \frac{X_{g_{\star}} \xi^{\star}}{\norm{\xi^{\star}}_{\epsilon_{g_{\star}}}^{D}}, \,
\xi^{\star} \!=\! \ST{(1 - \epsilon_{g_{\star}})
 \norm{X_{g_{\star}}^{\top}\frac{y}{\lambda_{\max}}}_{\epsilon_{g_{\star}}}} \!\!\!\!
 \left(\! X_{g_{\star}}^{\top}\frac{y}{\lambda_{\max}} \!\right).
\end{align*}
The sequence $(\theta_k)_{k \in \bbN}$ is also obtained thanks to Eq.~\eqref{eq:dual_feasible_point}.
%\jo{to check the definition of $\eta$}

We now demonstrate the efficiency of our method in both synthetic and real datasets described below. For comparison, we report actual computation time
to reach convergence up to a certain tolerance on the duality gap.

\textbf{Synthetic dataset:}
We use a common framework \citep{Tibshirani_Bien_Friedman_Hastie_Simon_Tibshirani12, Wang_Ye14} based on the model $y = X \beta + 0.01 \varepsilon$ where $\varepsilon
\sim \mathcal{N}(0, \Id_n)$, $X \in \bbR^{n \times p}$ follows a multivariate
normal distribution such that $\forall (i,j) \in [p]^2, \text{corr}(X_i, X_j) =
\rho^{|i-j|}$. We fix $n=100$ and break randomly $p = 10000$ in $1000$ groups
of size 10 and select $\gamma_1$ groups to be active and the others
are set to zero. In each of the selected groups, $\gamma_2$
coordinates are drawn such that $[\beta_g]_j = \sign(\xi) \times U$ where
$U$ is uniform in $[0.5, 10])$, $\xi$ uniform in $[-1, 1]$. %and $\alpha > 0$.
The results of this experiment are presented in Section~\ref{sec:performance_screening_rule}.

\begin{figure*}[!ht]
\subfigure[Proportion of active coordinate variables as a function of
parameters $(\lambda_t)$ and the number of iterations $K$.]
{\includegraphics[width=0.32\linewidth, keepaspectratio]{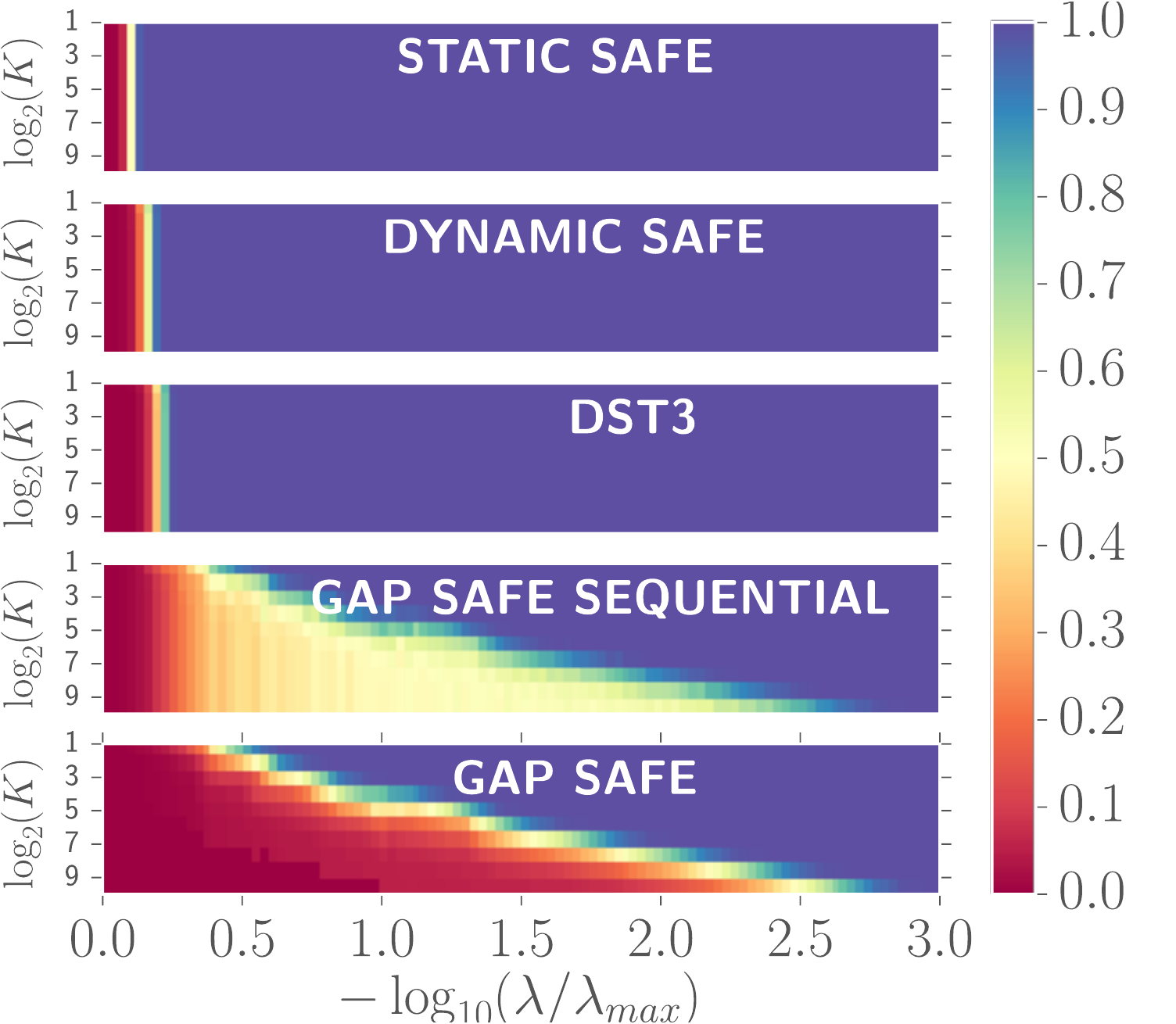}}
\hfill
\subfigure[Proportion of active group variables as a function of
parameters $(\lambda_t)$ and the number of iterations $K$.]
{\includegraphics[width=0.32\linewidth, keepaspectratio]{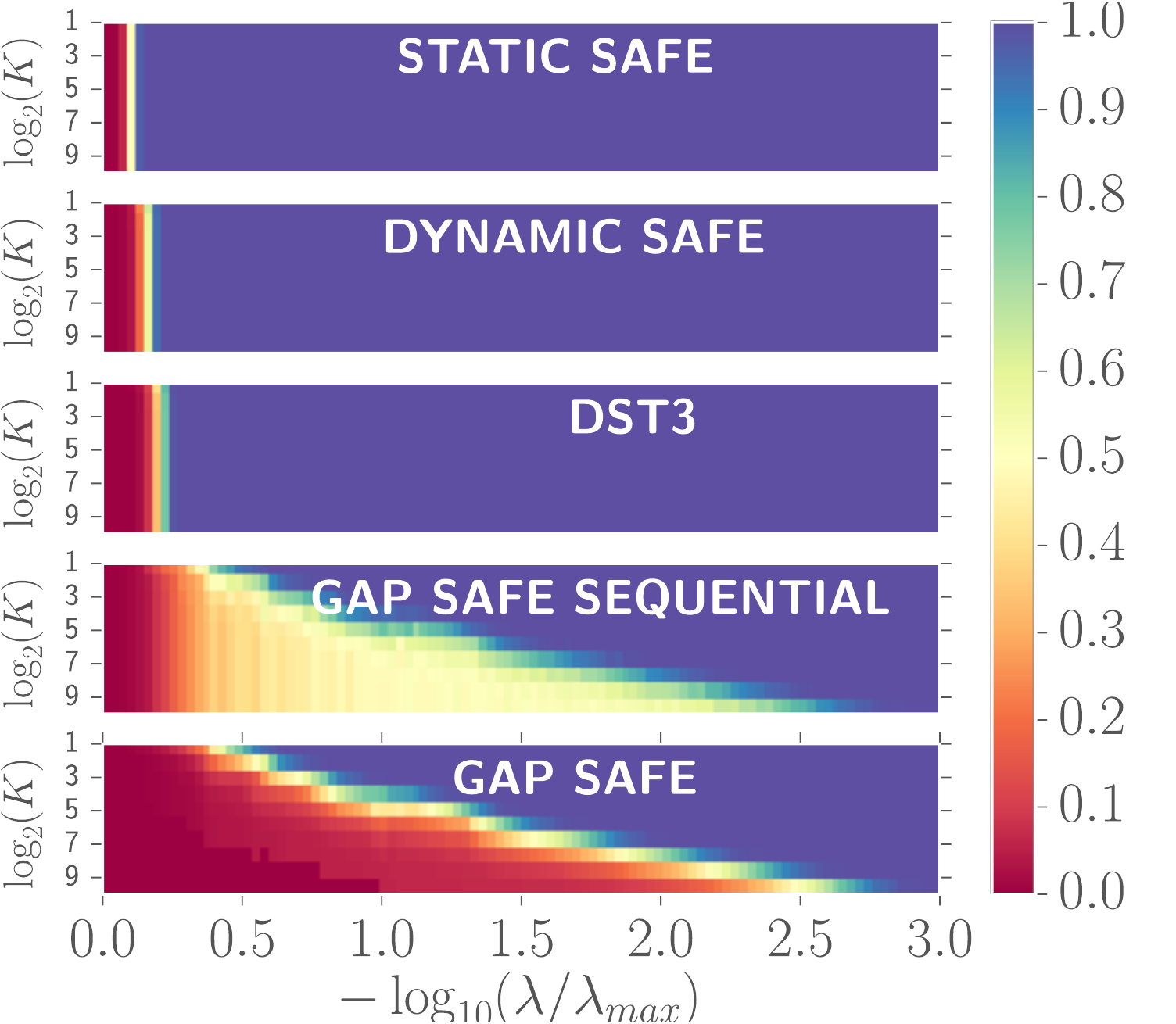}}
\hfill
\subfigure[Time to reach convergence as a function of increasing prescribed accuracy and using various screening strategies.\label{fig:lowcorr_comptime}]
{\includegraphics[width=0.32\linewidth, keepaspectratio]{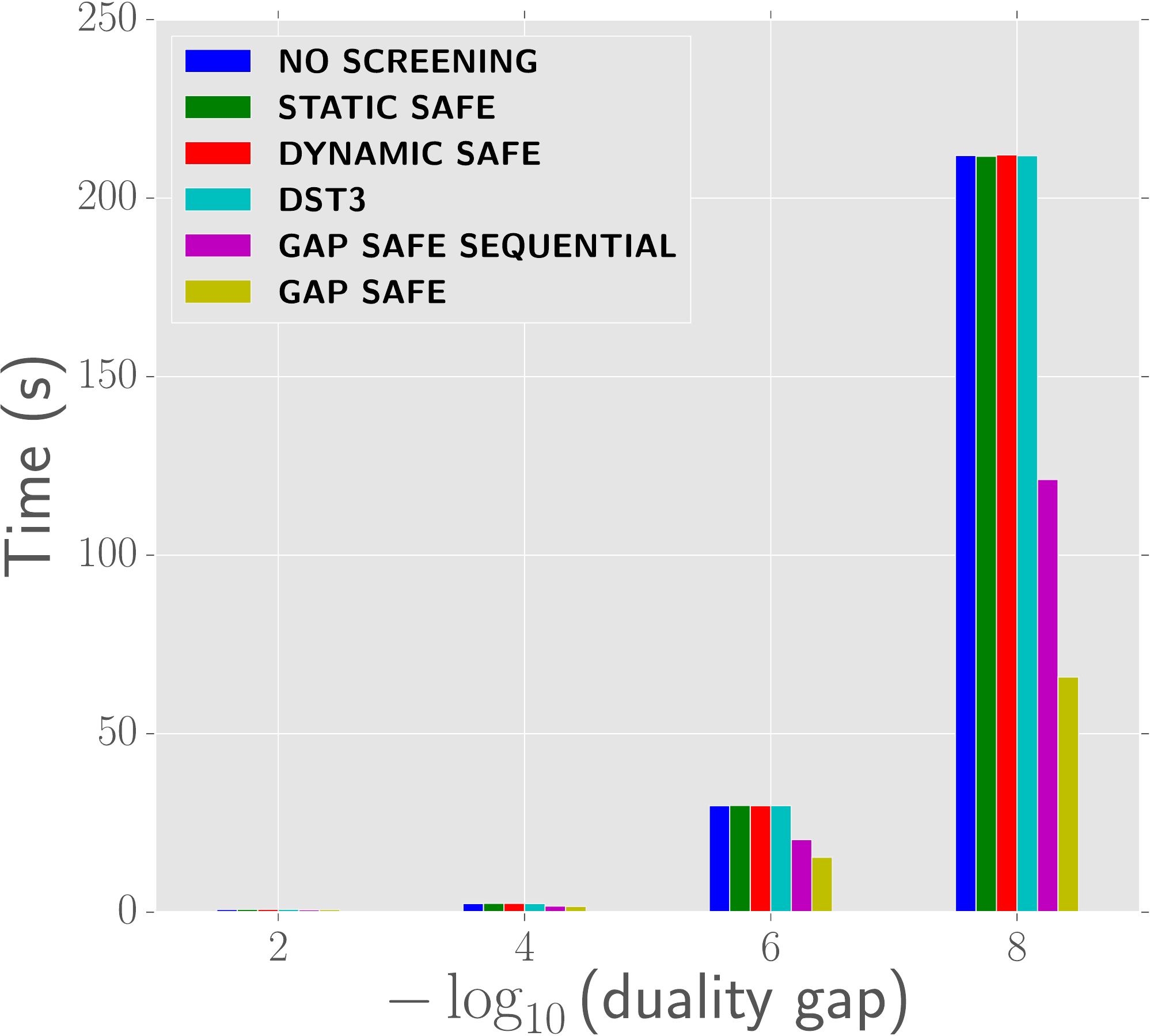}}
\caption{Experiments on a synthetic dataset
($\rho = 0.5, \gamma_1 = 10, \gamma_2 = 4, \tau = 0.2$).\label{fig:synth_dataset}}
\end{figure*}

%  \item
 \textbf{Real dataset: NCEP/NCAR  Reanalysis 1 \cite{Kalnay_Kanamitsu_Kistler_Collins_Deaven_Gandin_Iredell_Saha_White_Woollen_Others96}}
The dataset contains monthly means of climate data measurements spread across the globe in a
grid of $2.5^\circ \times 2.5^\circ$ resolutions (longitude and latitude $144
\times 73$) from $1948/1/1$ to $2015/10/31$ . Each grid point constitutes a group of
 $7$ predictive variables (\textit{Air Temperature, Precipitable water, Relative humidity, Pressure,
Sea Level Pressure, Horizontal Wind Speed}
and \textit{Vertical Wind Speed}) whose concatenation across time constitutes our design
matrix $X \in \bbR^{814 \times 73577}$. Such data have therefore a natural group structure.

\begin{figure*}[ht]
\subfigure[We show the prediction error for the \SGL path with $100$ values of $\lambda$ and
$11$ values of $\tau$. The best performance is achieved with $\tau^{\star} = 0.4$.\label{fig:cross_val}]
{\includegraphics[width=0.5\columnwidth,keepaspectratio]{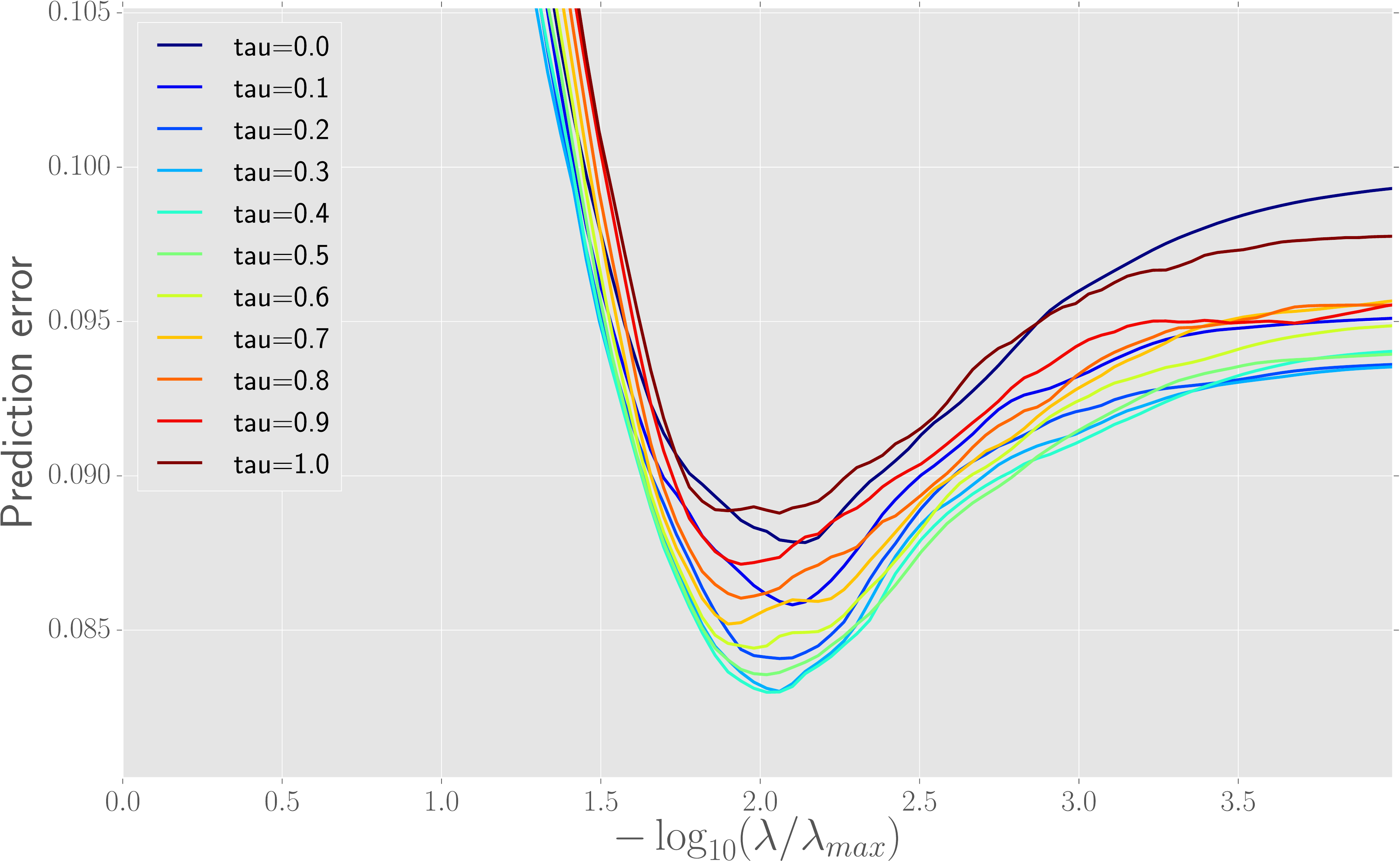}}
$\qquad$
\subfigure[We show the computation time to reach convergence as a function of the desired
accuracy on the dual gap.
The time includes the whole path over $(\lambda_t)_{t \in [T]}$ with $\delta = 2.5$ and
$\tau^{\star} = 0.4$.\label{fig:bench_time_clim}]
{\includegraphics[width=0.5\columnwidth, keepaspectratio]{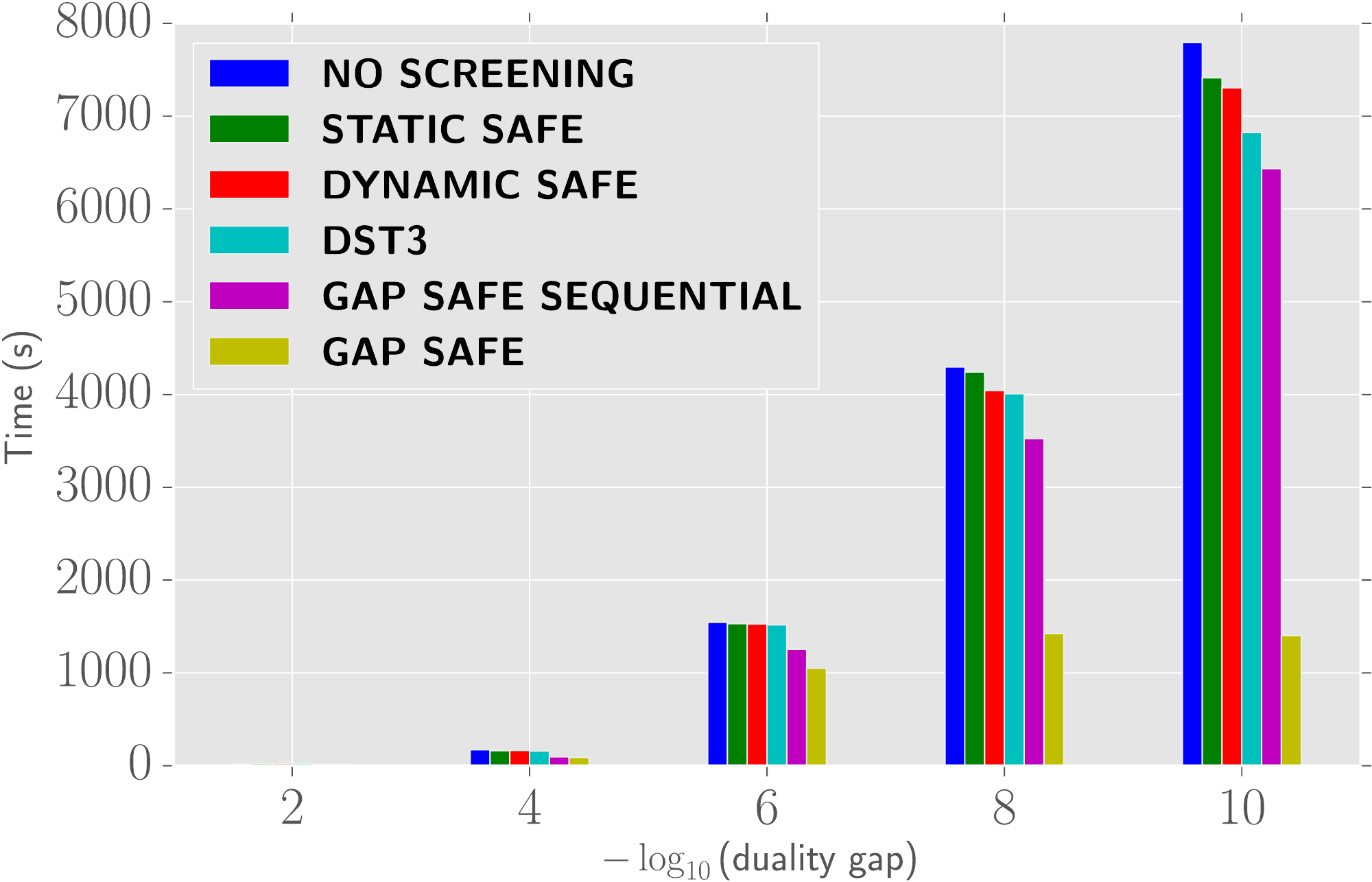}}
\caption{Experiments on NCEP/NCAR Reanalysis 1 dataset (n = 814, p = 73577).}
\end{figure*}

\begin{figure*}[ht]
\includegraphics[width=\columnwidth,keepaspectratio]{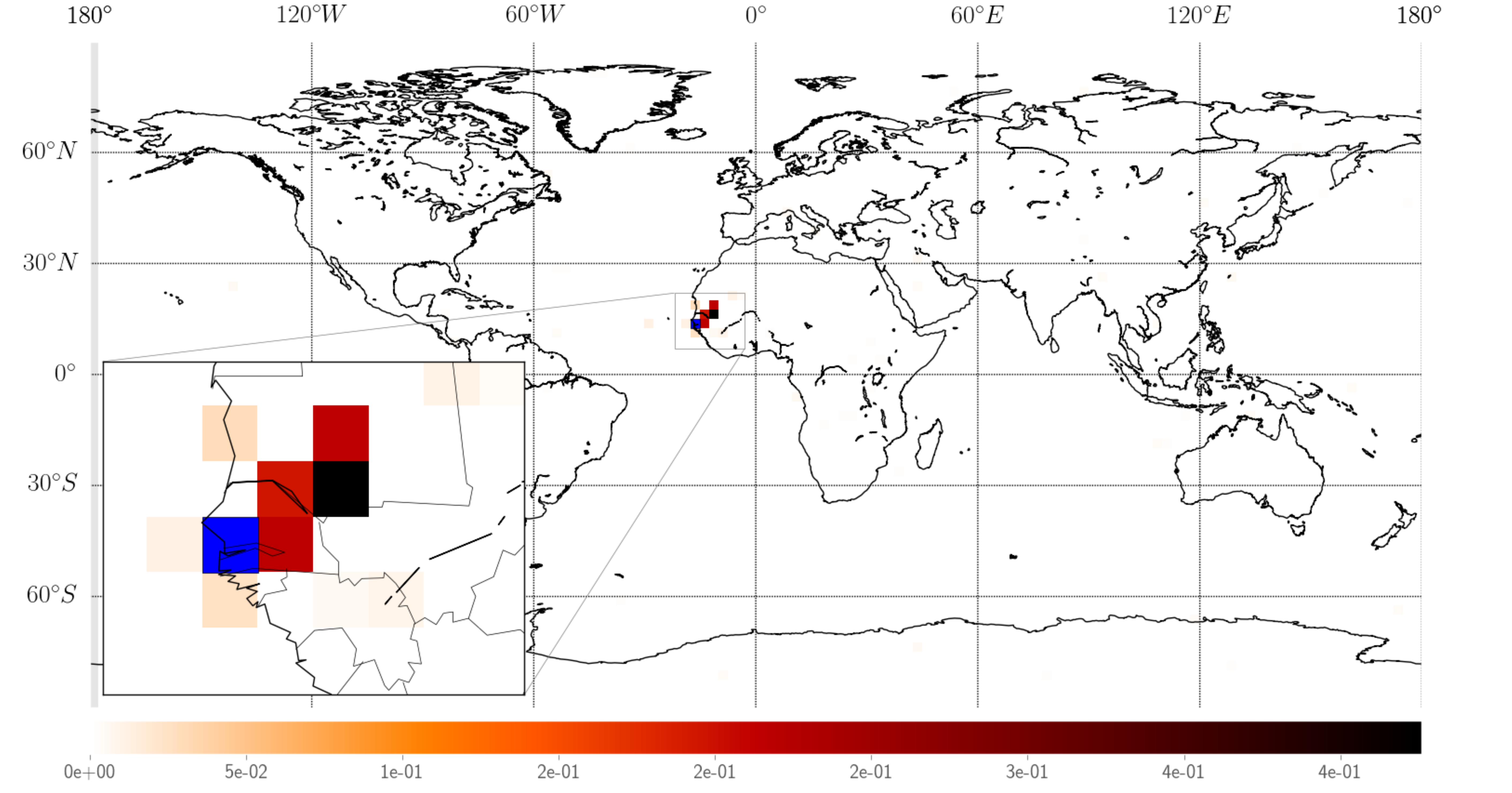}
\caption{Experiments on NCEP/NCAR Reanalysis 1 dataset (n = 814, p = 73577).
We show the active groups for the prediction of Air Temperature in a neighborhood of Dakar(location in blue). The regression coefficient are obtained by cross validation over
$100$ values of $\lambda$ and $11$ values of $\tau$. At each location, we present
the highest absolute value among the seven coefficients.}
\label{fig:support_map}
\end{figure*}

In our experiments, which aim to illustrate the computational benefit of the proposed method, we considered as target variable $y \in \bbR^{814}$, the values of \textit{Air Temperature} in a neighborhood of Dakar.
For preprocessing, we remove the seasonality and the trend present in the dataset. This is usually done in climate analysis to prevent some bias in the regression estimates. Similar data have been used in the past by \citet{chatterjee_Steinhaeuser_Banerjee_Chatterjee_Ganguly_12}, demonstrating that the
\SGL estimator is well suited for prediction in such climatology applications.
Indeed, thanks to the sparsity structure the estimates delineate via their support some predictive regions at the group level, as well as predictive feature via coordinate-wise screening.

We choose the parameter $\tau$ in the set $\{0, 0.1, \ldots, 0.9, 1\}$ by splitting in $50\%$ the observations and run a training-test validation procedure. For each value of $\tau$, we require a duality gap of $10^{-8}$ on the training part and pick the best one in term of prediction accuracy on the test part.
The result is displayed in Figure \ref{fig:cross_val}. Since the prediction error degrades increasingly
for $\lambda \leq \lambda_{\max}/10^{-2.5}$, we fix $\delta=2.5$ for the computational time benchmark
in Figure~\ref{fig:bench_time_clim}.

\subsection{Performance of the screening rules} \label{sec:performance_screening_rule}

In all our experiments, we observe that our proposed Gap Safe rule outperforms the other rules in term of computation time. On Figure~\ref{fig:lowcorr_comptime}, we can see that we need $65$s to reach convergence whereas others rules need up to $212$s at a precision of $10^{-8}$. A similar performance is observed on the real dataset (Figure~\ref{fig:bench_time_clim})
where we obtain up to a $5$x speed up
over the other rules.
The key reason behind this performance gain is the convergence
of the Gap Safe regions toward the dual optimal point as well as the efficient strategy to compute the screening rule. As shown in the
results presented on Figure~\ref{fig:synth_dataset}, our method still manages to screen out variables when $\lambda$ is small. It corresponds to low regularizations which lead to less sparse solutions but need to be explored during cross-validation.

In the climate experiments, the support map in Figure~\ref{fig:support_map} shows that the
most important coefficients are distributed in the vicinity of the target region
(in agreement with our intuition). Nevertheless, some active variables with small coefficients remain and cannot be screened out.

Note that we do not compare our method to the \textbf{TLFre} \citep{Wang_Ye14}, since this sequential rule requires the exact
knowledge of the dual optimal solution which is not
available in practice. As a consequence, one may discard active variables
which can prevent the algorithm from converging as shown in
\citep[Figure~4]{Ndiaye_Fercoq_Gramfort_Salmon15} for the Group-Lasso.
This issue still occurs with the method explored by \citet{Lee_Xing14} for overlapping groups.
% Furthermore, these screening rules are based on the fact that the dual optimal solution is a projection of $y/\lambda$ on the dual feasible set; which is not true for a general data fitting term.
% Hence, our method is much more flexible.

%!TEX root = ../icml.tex

\section{Conclusion} The recent GAP safe rules introduced in \cite{Fercoq_Gramfort_Salmon15,Ndiaye_Fercoq_Gramfort_Salmon15} for a wide range of regularized regression have shown great improvements in the reduction of computational burden specially in high dimension. A thorough investigation of the \SGL norm allows us to generalize the GAP safe rule to the \SGL problem. We give a new description of the dual feasible set by establishing a connection between the \SGL norm and the $\epsilon$-norm. This new point of view on the geometry of the problem helps providing an efficient algorithm to compute the dual norm and dual feasible points. Extending GAP safe rules on more general hierarchical regularizations \cite{Wang_Ye15}, is a possible direction for future research.

% We also extend to the \SGL, the safe screening rules introduced in\citet{ElGhaoui_Viallon_Rabbani12}, \citet{Xiang_Xu_Ramadge11} and
% \citet{Bonnefoy_Emiya_Ralaivola_Gribonval14} and finally demonstrate that
% our method is efficient and competitive.

\vskip 0.2in
\bibliography{references_all}
\bibliographystyle{icml2016}

%!TEX root = ../icml_supp.tex
\appendix
\newpage
\onecolumn

\section{Additional convexity and optimization tools}
In what follows we will use the dot product notation for any $x,x' \in \bbR^d$ we write $\langle x,x'\rangle=x^\top x'$.

We denote by $\iota_C$ the indicator function of a set $C$ defined as
\begin{equation}
\iota_C: \bbR^d \rightarrow \bbR, \quad
\iota_C (x) =
\begin{cases}
 0, & \text{ if } x \in C, \\
 +\infty, &\text{ otherwise.}
\end{cases}
\end{equation}

We denote by $f^*:\bbR^d \rightarrow \bbR$ the Fenchel conjugate of $f$ defined
for any $z \in \bbR^d$ by $f^* (z) = \sup_{w \in \bbR^d} w^\top z - f(w)$.

\begin{proposition}%[Fenchel conjugate of a norm]
(\citet[Prop. 1.4]{Bach_Jenatton_Mairal_Obozinski12})
The Fenchel conjugate of the norm $\Omega$ is given by
\begin{equation}
\Omega^{*}(\xi) =  \sup_{w \in \bbR^d} [\xi^{\top} w - \Omega(w)] =
\iota_{\mathcal{B}_{\Omega^D}} (\xi).
\end{equation}
\end{proposition}

\section{Proofs}

% {
% \renewcommand{\thetheorem}{\ref{prop:theoretical_screening_rules}}
\begin{repproposition}{prop:theoretical_screening_rules}[Theoretical screening rules]
The two levels of screening rules for the \SGL are: \\
\textbf{Feature level screening: }
\begin{equation*}
\forall j \in g, \,
|X_j^\top \ttheta{\lambda,\tau,w}| < \tau
\Longrightarrow \tbeta{\lambda,\tau,w}_j = 0.
\end{equation*}
\textbf{Group level screening: }
\begin{equation*}
\forall g \in \mathcal{G}, \,
\|\mathcal{S}_\tau(X_g^\top \ttheta{\lambda,\tau,w})\|
< (1-\tau) w_g \Longrightarrow \tbeta{\lambda,\tau,w}_g = 0.
\end{equation*}
\end{repproposition}

\begin{proof}
% of proposition \ref{prop:theoretical_screening_rules}

Let us consider $\tbeta{\lambda,\tau,w}_g \neq 0$, $g \in \mathcal{G}$. Then
combining the \textbf{subdifferential inclusion}
\eqref{eq:sub-differential_inclusion}, the subdifferential of the $\ell_2$-norm
\eqref{eq:sub-differential_2} and the decomposition of any dual feasible point
\eqref{rq:decomposition_of_dual_feasible_point}, we obtain :
\begin{align*}
X_{g}^{\top} \ttheta {\lambda,\tau,w} &= \tau v_g + (1-\tau) w_g
\frac{\tbeta{\lambda,\tau,w}_g}{\norm{\tbeta{\lambda,\tau,w}}}
\text{ where } v \in \partial \norm{\cdot}_1(\tbeta{\lambda,\tau,w}), \\
 X_{g}^{\top} \ttheta {\lambda,\tau,w} &=
\Pi_{\tau \mathcal{B}_{\infty}}(X_{g}^{\top} \ttheta {\lambda,\tau,w}) +
\ST{\tau} (X_{g}^{\top} \ttheta{\lambda,\tau,w}).
\end{align*}

So we can deduce that
$\ST{\tau} (X_{g}^{\top} \ttheta {\lambda,\tau,w}) \in (1-\tau) w_g
\left\{ \frac{\tbeta{\lambda,\tau,w}_g}{\norm{\tbeta{\lambda,\tau,w}_g}}
\right\}$.
Since $\ttheta {\lambda,\tau,w}$ is feasible
% \ref{rq:dual_norm_theta_star},
then
$\|\ST{\tau} (X_{g}^{\top} \ttheta {\lambda,\tau,w})\| < (1-\tau)w_g$ is
equivalent to
$\|\ST{\tau} (X_{g}^{\top} \ttheta {\lambda,\tau,w})\| \neq (1-\tau)w_g $
which implies, by contrapositive, that $\tbeta{\lambda,\tau,w}_g = 0$. Hence we
obtain the group level safe rule.
Furthermore, from the subdifferential of the
$\ell_1$-norm \eqref{eq:sub-differential_1}, we have:
\begin{equation*}\label{eq:Fermat_non_zero_case}
\forall j \in g, \, X_{j}^{\top} \ttheta{\lambda,\tau,w} \in
\begin{cases}
(1-\tau)w_g \left\{
\frac{\tbeta{\lambda,\tau,w}_j}{\|\tbeta{\lambda,\tau,w}\|} \right\} +
\tau \left\{ \sign(\tbeta{\lambda,\tau,w}_j) \right\},
&\text{ if } \tbeta{\lambda,\tau,w}_j \neq 0 , \\
[-\tau,\tau], &\text{ if } \tbeta{\lambda,\tau,w}_j = 0.
\end{cases}
\end{equation*}
Hence, if $\tbeta{\lambda,\tau,w}_j \neq 0$ then
$X_{j}^{\top} \ttheta{\lambda,\tau,w} = \sign(\tbeta{\lambda,\tau,w}_j) \left[
(1-\tau)w_g
\frac{|\tbeta{\lambda,\tau,w}_j|}{\|\tbeta{\lambda,\tau,w}\|} + \tau \right]$
and so $ |X_{j}^{\top} \ttheta{\lambda,\tau,w}| \geq \tau$.
By contrapositive, we obtain the feature level safe rule.
\end{proof}

\begin{repproposition}{prop:screening_bound}
For all group $g \in \mathcal{G}$ and $j \in g$,
\begin{equation}
\max_{\theta \in \mathcal{B}(\theta_c, r)} |X_j^\top \theta| \leq
|X_j^\top \theta_c| + r \|X_j\|.
\end{equation}

$\max_{\theta \in \mathcal{B}(\theta_c, r)} \| \ST{\tau} (X_g^\top \theta) \|$
is upper bounded by
\begin{equation}
\begin{cases}
 \|\ST{\tau} (X_g^\top \theta_c)\| + r \|X_g\| &\text{ if }
 \|X_g^\top \theta_c\|_{\infty} > \tau, \\
 (\|X_g^\top \theta_c\|_{\infty} + r \|X_g\| - \tau)_+ &\text{ otherwise}.
\end{cases}
\end{equation}
\end{repproposition}

\begin{figure}
\centering
\subfigure[$\mathcal{B}(\xi_c,\tilde r) \cap \tau \mathcal{B}_\infty \neq \emptyset$;
$\xi_c \in \tau  \mathring{\mathcal{B}_\infty}$]
{\includegraphics[width=.3\linewidth]{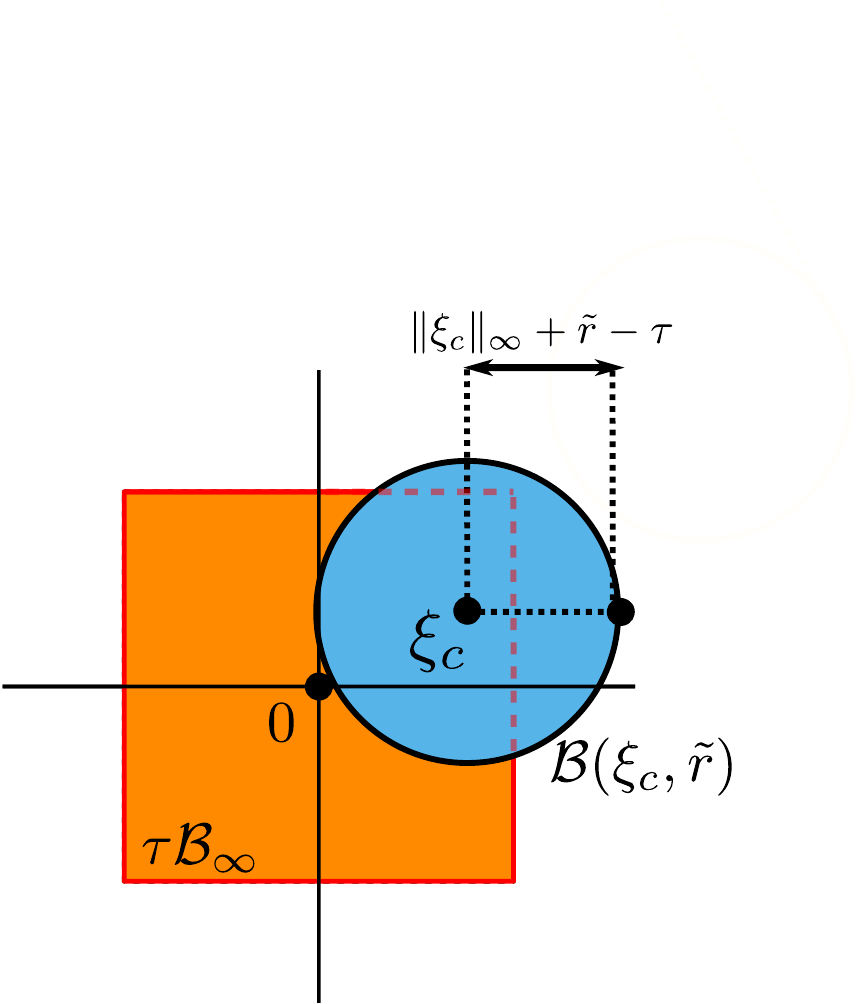}}
\subfigure[$\mathcal{B}(\xi_c, \tilde r) \subset \tau \mathcal{B}_\infty$]
{\includegraphics[width=.3\linewidth]{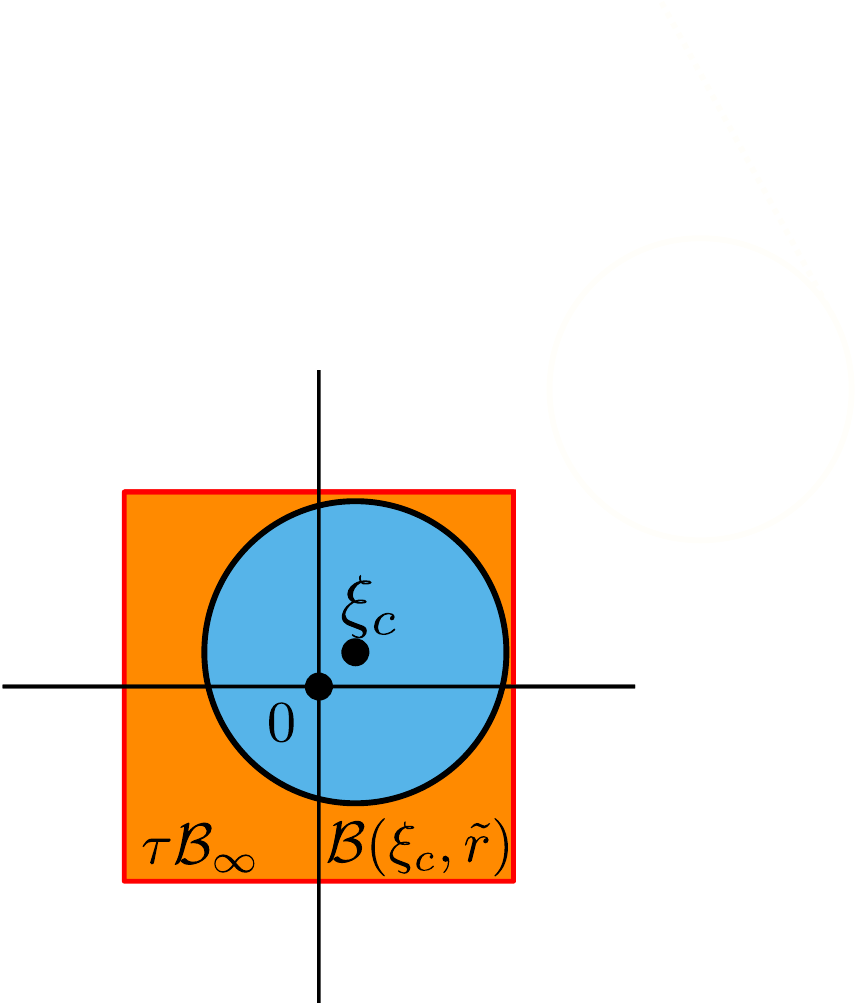}}
\subfigure[$\mathcal{B}(\xi_c, \tilde r) \cap \tau \mathcal{B}_\infty = \emptyset$;
$\xi_c \notin \tau \mathring{\mathcal{B}_\infty}$]
{\includegraphics[width=.3\linewidth]{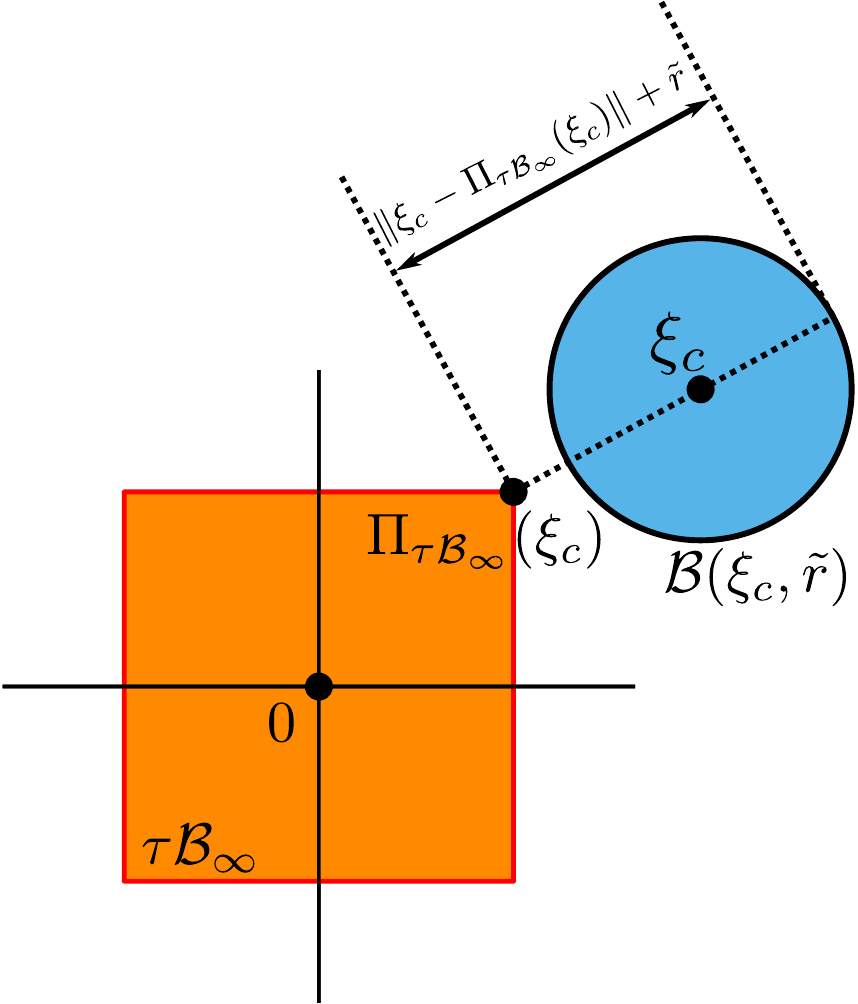}}
\end{figure}

\begin{proof}
$|X_{j}^{\top} \theta| \leq |[X_g^\top (\theta - \theta_c)]_j| + |X_j^\top \theta_c|
\leq r \|X_j\| + |X_j^\top \theta_c|
$ as soon as $\theta \in \mathcal{B}(\theta_c, r)$.

Since $\theta \in \mathcal{B}(\theta_c, r)$ implies that
$X_g^\top \theta \in \mathcal{B}(X_g^\top \theta_c, r \|X_g\|)$, we have
$\max_{\theta \in \mathcal{B}(\theta_c, r)} \|\ST{\tau} (X_g^\top \theta)\| \leq
\max_{\xi \in \mathcal{B}(\xi_c, \tilde r)} \|\ST{\tau} (\xi)\|$
where $\xi_c = X_{g}^{\top}\theta_c$ and $\tilde r = r \norm{X_j}$. From now,
we just have to show how to compute $\max_{\xi \in \mathcal{B}(\xi_c, \tilde r)}
\|\ST{\tau} (\xi)\|$.
\begin{itemize}
 \item In the case where $\xi_c \in \mathring{\tau \mathcal{B}_{\infty}}$, if
$\|\xi_c\|_{\infty} + \tilde r \leq \tau \,
( \text{ \ie } \mathcal{B}(\xi_c, \tilde r) \subset \tau \mathcal{B}_{\infty})$,
we have
$\Pi_{\tau \mathcal{B}_{\infty}} (\xi) = \xi$ and thus,
$\max_{\xi \in \mathcal{B}(\xi_c, \tilde r)} \|\ST{\tau}(\xi)\| =
\max_{\xi \in \mathcal{B}(\xi_c, \tilde r)}
\|\xi - \Pi_{\tau \mathcal{B}_{\infty}}(\xi)\| = 0.$

\item Otherwise if $\xi_c \in \mathring{\tau \mathcal{B}_{\infty}}$
and $\|\xi_c\|_{\infty} + \tilde r > \tau$, for any vector
$\xi \in \partial \mathcal{B}(\xi_c, \tilde r) \cap (\tau
\mathcal{B}_{\infty})^{c}$
and any vector $\tilde \xi \in \partial \tau \mathcal{B}_{\infty} \cap [\xi,
\xi_c]$,  $\|\xi - \Pi_{\tau \mathcal{B}_{\infty}}(\xi)\| \leq
\|\xi - \tilde \xi\| = \tilde r - \|\tilde \xi - \xi_c\|
$. Hence
% Indeed, it is trivial to see that a vector $\tilde \xi \in \tau
% \mathcal{B}_{\infty}$ cannot closer to $\xi$ than the projection. Hence
\begin{equation*}
 \max_{\xi \in \mathcal{B}(\xi_c, \tilde r)}
 \|\xi - \Pi_{\tau \mathcal{B}_{\infty}}(\xi) \| \leq
\max_{\underset{\tilde \xi\in \partial \tau \mathcal{B}_{\infty} \cap [\xi,
\xi_c]}{\xi \in \partial \mathcal{B}(\xi_c, \tilde r)\cap (\tau
\mathcal{B}_{\infty})^{ c}}
}
\tilde r - \|\tilde \xi - \xi_c\|
\leq \tilde r - \min_{\xi \in \partial \tau \mathcal{B}_{\infty}} \|\xi - \xi_c\|
 = \tilde r - \tau + \|\xi_c\|_{\infty}.
\end{equation*}
% \eugenedo{$\|c\|_{\infty} e_{j^{\star}}$ attains the bound} \\
This upper bound is attained. Indeed,
$ \max_{\theta \in \mathcal{B}(\xi_c, \tilde r)}
 \|\xi - \Pi_{\tau \mathcal{B}_{\infty}}(\xi) \|
= \tilde r - \|\Pi_{\tau \mathcal{B}_{\infty}}( \hat \xi) - \xi_c\|
= \tilde r - \tau + \|\xi_c\|_{\infty}
$ where $\hat \xi$ is a vector in
$\partial \mathcal{B}(\xi_c, \tilde r)$ such that
$\Pi_{\tau \mathcal{B}_{\infty}}( \hat \xi) = \xi_c + e_{j^{\star}} (\tau -
\|\xi_c\|_{\infty})$ and $j^{\star} \in \argmax_{j \in [p]} |(\xi_c)_j|$.

\item If $\xi_c \notin \mathring{\tau \mathcal{B}_{\infty}}$, since
 the projection operator on a convex set is a contraction, we have
\begin{equation*}
 \forall \xi \in \partial \mathcal{B}(\xi_c, \tilde r), \,
\|\xi - \Pi_{\tau \mathcal{B}_{\infty}}(\xi)\|  \leq
\|\xi - \Pi_{\tau \mathcal{B}_{\infty}}(\xi_c)\|
\leq \|\xi_c - \Pi_{\tau \mathcal{B}_{\infty}}(\xi_c)\| + \|\xi - \xi_c\|
= \|\xi_c - \Pi_{\tau \mathcal{B}_{\infty}}(\xi_c)\| + \tilde r.
\end{equation*}

Moreover, it is straightforward to see that the vector
$\tilde \xi :=  \tilde \gamma \xi_c + (1 - \tilde \gamma) \Pi_{\tau
\mathcal{B}_{\infty}}(\xi_c)$ where
$\tilde \gamma = 1 + \frac{\tilde r}{\norm{\xi_c} + \norm{\Pi_{\tau
\mathcal{B}_{\infty}}(\xi_c)}}$ belongs to
$\partial \mathcal{B}(\xi_c, \tilde r)$; it verifies $\Pi_{\tau
\mathcal{B}_{\infty}}(\xi_c) =
\Pi_{\tau \mathcal{B}_{\infty}}(\tilde \xi)$
and it attains this bound. \qedhere
\end{itemize}
\end{proof}

\begin{reptheorem}{th:GAP_Safe_sphere}[Safe radius]
For any $\theta \in \dualomegasgl$ and any $\beta\in \bbR^p$, one has
$ \ttheta{\lambda,\tau, w} \in
\mathcal{B}\left(\theta,\bestR[\lambda,\tau]{\beta}{\theta}\right),$ for
\begin{align*}
% \text{ with }
\bestR[\lambda,\tau]{\beta}{\theta}&=\sqrt{\frac{2(P_{\lambda,\tau,w}(\beta)-
D_{\lambda}(\theta))}{\lambda^2}},
\end{align*}
\ie the aforementioned ball is a safe
region for the \SGL problem.
\end{reptheorem}

\begin{proof}
% in appendix.
By weak duality,
$ \forall \beta \in \bbR^p, \,D_{\lambda}(\ttheta{\lambda,\tau,w}) \leq
P_{\lambda,\tau,w}(\beta)$.
Then, note that the dual objective function \eqref{eq:general_primal_problem}
is $\lambda^2$-strongly concave. This implies:
\begin{equation*}
\forall (\theta, \theta') \in \dualomegasgl \times \dualomegasgl, \quad
D_{\lambda}(\theta) \leq
D_{\lambda}(\theta') + \nabla D_{\lambda}(\theta')^\top
(\theta-\theta') -\frac{\lambda^2}{2}
\norm{\theta - \theta'}^2.
\end{equation*}

Moreover, since $\ttheta{\lambda,\tau,w}$ maximizes the concave function
$D_{\lambda}$,  the following inequality holds true:

\begin{equation*}
\forall \, \theta \in \dualomegasgl,\quad \nabla
D_{\lambda}(\ttheta{\lambda,\tau,w})^\top (\theta-\ttheta{\lambda,\tau,w})
\leq 0.
\end{equation*}
Hence, we have for all $\theta \in \dualomegasgl$ and $\beta\in \bbR^p$:
\begin{align*}
\frac{\lambda^2}{2} \|\theta-\ttheta{\lambda,\tau,w}\|^2 &\leq
D_{\lambda}(\ttheta{\lambda,\tau,w}) - D_{\lambda}(\theta) \\
&\leq P_{\lambda,\tau,w}(\beta)- D_{\lambda}(\theta). \qedhere
\end{align*}
\end{proof}

% \Prop \ref{prop:dual_convergence} :
%  If $\lim_{k \to \infty} \beta_k = \tbeta{\lambda, \tau, w} $, then
%  $\lim_{k \to \infty} \theta_k = \ttheta{\lambda, \tau, w} $.
%
% \begin{proof}
% % of proposition \ref{prop:dual_convergence}
% \begin{align*}
% \norm{\theta_k-\ttheta{ \lambda, \tau,w}} &= \norm{\frac{1}{\normsgl^D(X^\top
% \rho_k)}(y-X\beta_k) - \frac{1}{\lambda} (y-X \tbeta{\lambda, \tau, w})} \\
% &= \norm{ \left(\frac{1}{\normsgl^D(X^\top \rho_k)}
% -\frac{1}{\lambda}\right)(y-X\beta_k) -
% \frac{(X \tbeta{\lambda, \tau,w} - X\beta_k)}{\lambda}}\\
% &\leq \left|\frac{1}{\normsgl^D(X^\top \rho_k)} -\frac{1}{\lambda}\right|
% \norm{y-X \beta_k} +
% \norm{\frac{X \tbeta{\lambda, \tau,w} - X \beta_k}{\lambda}}.
% \end{align*}
% If $\beta_k \rightarrow \tbeta{\lambda, \tau, w}$, then the last term converges
% to zero and
% $\normsgl^D(X^\top \rho_k) \rightarrow
% \normsgl^D(X^\top (y - X \tbeta{\lambda,\tau,w})=
% \lambda \normsgl^D(X^\top \ttheta{\lambda, \tau, w}) = \lambda$ since
% $\normsgl^D(X^\top \ttheta{\lambda, \tau, w}) = 1$.
% \end{proof}

% proof with lambda_max

\begin{repproposition}{prop:dual_convergence}
 If $\lim_{k \to \infty} \beta_k = \tbeta{\lambda, \tau, w} $, then
 $\lim_{k \to \infty} \theta_k = \ttheta{\lambda, \tau, w} $.
\end{repproposition}

\begin{proof}
% of proposition \ref{prop:dual_convergence}
Let $\alpha_k = \max(\lambda, \normsgl^D(X^\top \rho_k))$ and recall that
$\rho_k = y - X\beta_k$. We have :
\begin{align*}
\norm{\theta_k-\ttheta{ \lambda, \tau,w}} &=
\norm{\frac{1}{\alpha_k}(y-X\beta_k) - \frac{1}{\lambda} (y-X \tbeta{\lambda, \tau, w})} \\
&= \norm{ \left(\frac{1}{\alpha_k}
-\frac{1}{\lambda}\right)(y-X\beta_k) -
\frac{(X \tbeta{\lambda, \tau,w} - X\beta_k)}{\lambda}}\\
&\leq \left|\frac{1}{\alpha_k} -\frac{1}{\lambda}\right|
\norm{y-X \beta_k} +
\norm{\frac{X \tbeta{\lambda, \tau,w} - X \beta_k}{\lambda}}.
\end{align*}
If $\beta_k \rightarrow \tbeta{\lambda, \tau, w}$, then $\alpha_k \rightarrow
\max(\lambda, \normsgl^D(X^\top (y - X \tbeta{\lambda,\tau,w}))=
\max(\lambda, \lambda \normsgl^D(X^\top \ttheta{\lambda, \tau, w})) = \lambda$ since
$y - X \tbeta{\lambda,\tau,w} = \lambda \ttheta{\lambda, \tau, w}$ thanks to the \textbf{link-equation}
\eqref{eq:link_equation} and since $\ttheta{\lambda, \tau, w}$ is feasible \ie
$\normsgl^D(X^\top \ttheta{\lambda, \tau, w}) \leq 1$. Hence, both terms in the previous inequality converge to zero.
\end{proof}

\begin{repproposition}{prop:convergence_regions}
Let $(\mathcal{R}_k)_{k \in \bbN}$ be a sequence of safe regions whose diameters converge to 0. Then,
$\displaystyle \lim_{k \rightarrow \infty}
\mathcal{A}_{\text{groups}}(\mathcal{R}_k) = \mathcal{E}_{\text{groups}}$ and
$\displaystyle \lim_{k \rightarrow \infty}
\mathcal{A}_{\text{features}}(\mathcal{R}_k) = \mathcal{E}_{\text{features}}$.
\end{repproposition}

\begin{proof}
We proceed by double inclusion. First let us prove that $\exists k_0$ s.t. $\forall k \geq k_0, \mathcal{A}_{\text{groups}}(\mathcal{R}_k) \subset \mathcal{E}_{\text{groups}}$. Indeed, since the diameter of $\mathcal{R}_k$ converges to zero, for any $\epsilon > 0$ there exist $k_0 \in \bbN, \forall k \geq k_0, \forall \theta \in \mathcal{R}_k, \| \theta - \ttheta{\lambda,\tau,w} \| \leq \epsilon$.
The triangle inequality implies that $\forall g \notin \mathcal{E}_{\text{groups}}$, $\|\ST{\tau}(X_{g}^{\top} \theta)\| \leq \|\ST{\tau}(X_{g}^{\top} \theta) - \ST{\tau}(X_{g}^{\top} \ttheta{\lambda,\tau,w})\| + \|\ST{\tau}(X_{g}^{\top} \ttheta{\lambda,\tau,w})\|$.
Since the soft-thresholding operator is $1$-Lipschitz, we have:
\begin{equation*}
\|\ST{\tau}(X_{g}^{\top} \theta)\| \leq \|X_g(\theta - \ttheta{\lambda,\tau,w})\| + \|\ST{\tau}(X_{g}^{\top} \ttheta{\lambda,\tau,w})\| \leq \epsilon \|X_g\| + \|\ST{\tau}(X_{g}^{\top} \ttheta{\lambda,\tau,w})\|,
\end{equation*}
as soon as $k \geq k_0$. Moreover, $\forall g \notin
\mathcal{E}_{\text{groups}}$,
\begin{equation*}
 \|\ST{\tau}(X_{g}^{\top} \theta)\|\leq \max_{g \notin \mathcal{E}_{\text{groups}}} \|\ST{\tau}(X_{g}^{\top} \theta)\| \leq \epsilon \max_{g \notin \mathcal{E}_{\text{groups}}} \|X_g\| + \max_{g \notin \mathcal{E}_{\text{groups}}} \|\ST{\tau}(X_{g}^{\top} \ttheta{\lambda,\tau,w})\|.
\end{equation*}
It suffices to choose $\epsilon$ such that
$$\epsilon \max_{g \notin \mathcal{E}_{\text{groups}}} \|X_g\| + \max_{g \notin \mathcal{E}_{\text{groups}}} \|\ST{\tau}(X_{g}^{\top} \ttheta{\lambda,\tau,w})\| < (1-\tau) w_g, $$
that is to say $\epsilon < \frac{(1-\tau)w_g - \max_{g \notin \mathcal{E}_{\text{groups}}} \|\ST{\tau}(X_{g}^{\top} \ttheta{\lambda,\tau,w})\|}{\max_{g \notin \mathcal{E}_{\text{groups}}} \|X_g\|}$,
to remove the group $g$. For any $k \geq k_0, \,
\mathcal{E}_{\text{groups}}^{c} = \{ g \in \mathcal{G}: \|\mathcal{S}_{\tau}(X_{g}^{\top} \ttheta{\lambda})\| < (1-\tau)w_g \} \subset \mathcal{A}_{\text{groups}}(\mathcal{R}_k)^c$, the set of variables removed by our screening rule.
This proves the first inclusion.

Now we show that $\forall k \in \bbN, \mathcal{A}_{\text{groups}}(\mathcal{R}_k) \supset \mathcal{E}_{\text{groups}}$. Indeed, for all $g^{\star} \in \mathcal{E}_{\text{groups}}$, $\|\ST{\tau}(X_{g^{\star}}^{T} \ttheta{\lambda,\tau,w})\| = (1-\tau)w_{g^{\star}}$. Since for all $k$ in $\bbN$, $\ttheta{\lambda,\tau,w} \in \mathcal{R}_k$ then $\underset{\theta \in \mathcal{R}_k}{\max} \|\ST{\tau}(X_{g}^{\top} \theta)\| \geq \|\ST{\tau}(X_{g^{\star}}^{T} \ttheta{\lambda,\tau,w}) \| = (1-\tau) w_{g^{\star}}$ hence the second
inclusion holds.

We have shown that that $\forall k \geq k_0$
$\mathcal{A}_{\text{groups}}(\mathcal{R}_k) =
\mathcal{E}_{\text{groups}}$ and so
$\mathcal{A}_{\text{features}}(\mathcal{R}_k) \subset
\bigcup_{g \in \mathcal{E}_{\text{groups}}}
\left\{ j \in g: \, \max_{\theta \in \mathcal{R}_k}
|X_{j}^{\top} \theta| \geq \tau \right\}$.
Moreover, the same reasoning yields $\forall g \in \mathcal{G}$,
$\left\{ j \in g: \, \max_{\theta \in \mathcal{R}_k}
|X_{j}^{\top} \theta| \geq \tau \right\} \subset
\left\{j \in g: \,|X_{j}^{\top} \ttheta{\lambda,\tau,w}| \geq \tau \right\}$.
Hence $\forall k \geq k_0, \mathcal{A}_{\text{features}}(\mathcal{R}_k) \subset
\mathcal{A}_{\text{features}}$. The reciprocal inclusion is straightforward.
\end{proof}

\begin{repproposition}{prop:properties_of_sgl}.
For all group $g$ in $\mathcal{G}$, let
% \tabularnewline
$\epsilon_g := \frac{(1 - \tau)w_g}{\tau + (1 - \tau)w_g}$
 then the \SGL norm satisfies the following properties:  for any vectors
$\beta$ and $\xi$ in $\bbR^p$
\vspace{-0.2cm}
\begin{align}
&\normsgl(\beta) = \sum_{g \in \mathcal{G}} (\tau + (1 - \tau) w_g)
\|\beta_g\|_{\epsilon_g}^{D} \\
% \label{eq:dual_norm_computation}
% = \sum_{g \in \mathcal{G}} \frac{(1 - \tau) w_g}{\epsilon_g}
% \|\beta_g\|_{\epsilon_g}^{D}, \\
&\normsgl^{D}(\xi) = \max_{g \in \mathcal{G}}
\frac{\norm{\xi_g}_{\epsilon_g}}{\tau + (1 - \tau) w_g}. \\
% \max_{g \in \mathcal{G}} \frac{(1- \tau) w_g}{\epsilon_g}
% \|\xi_g\|_{\epsilon_g},
% \label{eq:unit_dual_ball}
%&\iota_{\mathcal{B}_{\normsgl^D}} (\xi) =
%\sum_{g \in\mathcal{G}} \iota_{(1 - \tau) w_g \mathcal{B}}
%\left(\xi_g - \Pi_{\tau \mathcal{B}_\infty}(\xi_g)\right), \\
&\mathcal{B}_{\normsgl^D} =
\big\{\xi \in \bbR^p : \forall g \in \mathcal{G},
\|\ST{\tau}(\xi_g)\| \leq (1 - \tau)w_g \big\}
\end{align}
The subdifferential $\partial\normsgl(\beta)$ of the norm $\normsgl$ at $\beta$
is given by
\begin{equation*}
\bigg\{x \in \bbR^p: \forall g \in
\mathcal{G}, x_g \in \tau \partial \|\cdot\|_{1}(\beta_g) +
(1 - \tau) w_g \partial \|\cdot\|(\beta_g) \bigg\}
\end{equation*}
\end{repproposition}

\begin{proof}
% of proposition \ref{prop:properties_of_sgl}
\begin{align*}
\forall \beta \in \bbR^p, \,
 \Omega(\beta) &= \tau \|\beta\|_1 +
 (1 - \tau)\sum_{g \in \mathcal{G}} w_g \|\beta_g\|
 = \sum_{g \in \mathcal{G}}
\big(\tau \|\beta_g\|_1 +  (1 - \tau) w_g \|\beta_g\|\big) \\
 &= \sum_{g \in \mathcal{G}} (\tau + (1 - \tau) w_g)
 \left[\frac{\tau}{\tau + (1 - \tau) w_g} \|\beta_g\|_1 +
 \frac{(1 - \tau) w_g}{\tau + (1 - \tau) w_g} \|\beta_g\|
 \right] \\
 &=\sum_{g \in \mathcal{G}} (\tau + (1 - \tau) w_g)
\left[(1 - \epsilon_g) \|\beta_g\|_1 + \epsilon_g \|\beta_g\|\right]
= \sum_{g \in \mathcal{G}} (\tau + (1 - \tau) w_g)
\|\beta_g\|_{\epsilon_g}^{D}
\end{align*}

The definition of the dual norm reads $\displaystyle \Omega^D(\xi)= \max_{\beta:\Omega(\beta)\leq 1 } \beta^\top \xi$, and solving this problem yields:
\begin{align*}
 \Omega^D(\xi) &= \sup_{\beta: \Omega(\beta) \leq 1} \langle \beta, \xi \rangle
= \sup_{\beta} \inf_{\mu > 0}
\langle \beta, \sum_{g \in \mathcal{G}} \xi_g \rangle -
\mu \left(\sum_{g \in \mathcal{G}} \Omega_g(\beta_g) - 1 \right) \\
&= \inf_{\mu > 0} \left\{
\sum_{g \in \mathcal{G}} \sup_{\beta_g} \left[
\langle \beta_g, \xi_g \rangle - \mu \Omega_g(\beta_g) \right]
+ \mu \right\} \\
&= \inf_{\mu > 0} \left\{ \sum_{g \in \mathcal{G}} \mu
\Omega_{g}^{*}\left(\frac{\xi_g}{\mu}\right) + \mu \right\} =
\inf_{\mu > 0} \left\{ \sum_{g \in \mathcal{G}}
\iota_{\mathcal{B}_{\Omega_{g}^{D}}} \left(\frac{\xi_g}{\mu}\right) +
\mu \right\} \\
&= \inf_{\mu > 0} \left\{ \max_{g \in \mathcal{G}}
\iota_{\mathcal{B}_{\Omega_{g}^{D}}} \left(\frac{\xi_g}{\mu}\right) +
\mu \right\} =
\max_{g \in \mathcal{G}} \inf_{\mu > 0} \left\{
\Omega_{g}^{*} \left(\frac{\xi_g}{\mu}\right) + \mu \right\} \\
&= \max_{g \in \mathcal{G}} \inf_{\mu > 0} \sup_{\beta_g}
\langle \beta_g, \frac{\xi_g}{\mu} \rangle - \Omega_g(\beta_g) + \mu
\underset{\mu u_g = \beta_g}{=}
\max_{g \in \mathcal{G}} \inf_{\mu > 0} \sup_{u_g}
\langle u_g, \xi_g \rangle - \mu(\Omega_g(u_g) - 1) \\
&= \max_{g \in \mathcal{G}} \sup_{u_g: \Omega_g(u_g) \leq 1} \langle u_g,
\xi_g \rangle
= \max_{g \in \mathcal{G}} \sup_{u_g} \, \langle u_g, \xi_g \rangle \quad
\text{ s.t. } (\tau + (1-\tau)w_g) \norm{u_g}_{\epsilon_g}^{D} \leq 1 \\
&= \max_{g \in \mathcal{G}} \sup_{u_g: \Omega_g(u_g) \leq 1} \langle u_g,
\xi_g \rangle =
\max_{g \in \mathcal{G}} \sup_{{u'}_g:
\norm{{u'}_g}_{\epsilon_g}^{D} \leq 1}
\langle \frac{{u'}_g}{\tau + (1-\tau)w_g}, \xi_g \rangle =
\max_{g \in \mathcal{G}} \frac{\norm{\xi_g}_{\epsilon_g}}{\tau + (1-\tau)w_g}.
\end{align*}

We recall here the proof of \citet{Wang_Ye14} for the sake of completeness.
First let us write $\Omega(\beta) = \Omega_1(\beta) + \Omega_2(\beta)$, where
$\Omega_1(\beta) = \tau \norm{\beta}_1$ and
$\Omega_2(\beta) = (1 - \tau) \sum_{g \in \mathcal{G}} w_g \norm{\beta_g}_2$.
Since $\Omega_1$ and $\Omega_2$ are continuous everywhere, we have
(see \citet[Theorem 1]{Hiriart-Urruty06}):
$\Omega^{*}(\xi) = (\Omega_1 + \Omega_2)^{*}(\xi) =
 \min_{a + b = \xi} [\Omega_{1}^{*}(a) + \Omega_{2}^{*}(b)] =
 \min_{a} [\Omega_{1}^{*}(a) + \Omega_{2}^{*}(\xi - a)]$,
 which is also the inf-convolution (see \citet[Chapter
12]{Bauschke_Combettes11}) of these two norms.
Using the Fenchel conjugate of the $\ell_1$ norm ($\Omega_1^* = \iota_{\tau \mathcal{B}_{\infty}}$) and of the $\ell_2$ norm ($\Omega_2^*=\iota_{\mathcal{B}}$), we have
% \begin{align*}
% \Omega_{1}^{*}(\xi) &= \sup_{\beta} \langle \beta, \xi \rangle -
% \tau \norm{\beta}_1 =
% \tau \sup_{\beta} \langle \beta, \frac{\xi}{\tau} \rangle - \norm{\beta}_1 =
% \iota_{\tau \mathcal{B}_{\infty}} (\xi) \text{ and } \\
% \Omega_{2}^{*}(\xi) &= \sum_{g \in \mathcal{G}} (1 - \tau)w_g
% \sup_{\beta} \langle \beta_g, \frac{\xi}{(1 - \tau)w_g} \rangle -
% \norm{\beta_g}_2 = \sum_{g \in \mathcal{G}} \iota_{\mathcal{B}}
% \left(\frac{\xi_g}{(1 - \tau)w_g}\right). \text{ Hence, } \\
% \Omega^{*}(\xi) &= \sum_{g \in \mathcal{G}} \min_{a_g}
% \iota_{\tau \mathcal{B}_{\infty}} (a_g) +
% \iota_{\mathcal{B}} \left(\frac{\xi_g - a_g}{(1 - \tau)w_g}\right) =
% \sum_{g \in \mathcal{G}}
% \iota_{\mathcal{B}} \left(\frac{\xi_g -
% \Pi_{\tau \mathcal{B}_{\infty}}(\xi_g)}{(1 - \tau)w_g}\right).
% \end{align*}
\begin{equation*}
\Omega^{*}(\xi) = \sum_{g \in \mathcal{G}} \min_{a_g}
\iota_{\tau \mathcal{B}_{\infty}} (a_g) +
\iota_{\mathcal{B}} \left(\frac{\xi_g - a_g}{(1 - \tau)w_g}\right) =
\sum_{g \in \mathcal{G}}
\iota_{\mathcal{B}} \left(\frac{\xi_g -
\Pi_{\tau \mathcal{B}_{\infty}}(\xi_g)}{(1 - \tau)w_g}\right).
\end{equation*}

Hence the indicator of the unit dual ball is
$ \iota_{\mathcal{B}_{\Omega^D}} (\xi) =
\sum_{g \in\mathcal{G}} \iota_{(1 - \tau) w_g \mathcal{B}}
\left(\xi_g - \Pi_{\tau \mathcal{B}_\infty}(\xi_g)
\right)$ and using
$\ST{\tau}(\xi_g) = \xi_g - \Pi_{\tau \mathcal{B}_\infty}$, we have:
\begin{equation*}
\mathcal{B}_{\Omega^D} =
\big\{\xi \in \bbR^p: \Omega^D(\xi) \leq 1\big\}=
\big\{\xi \in \bbR^p: \forall g \in \mathcal{G}, \,
\|\ST{\tau}(\xi_g)\| \leq (1 - \tau)w_g \big\}.
\end{equation*}
\end{proof}

\begin{repproposition}{th:sparse_groupe_burdakov}.
For $\alpha \in [0,1], R\geq0$ and $x\in\bbR^d$, the equation $\sum_{j=1}^d
\ST{\nu \alpha}(x_j)^2 = (\nu R)^2$ has a unique
solution $\nu \in \bbR_+$, denoted by $\Lambda(x,\alpha,R)$ and that can be
computed in $O(d \log d)$ operations in worst case.
\end{repproposition}

\begin{proof}
Dividing by $\nu^2$, which is positive as soon as $x \neq 0$, we get that $\sum_{j=1}^d \ST{\nu \alpha}(x_j)^2 = (\nu R)^2$
is equivalent to $\sum_{j=1}^d \ST{\alpha}(x_j/\nu)^2 = R^2$.
Note that %\jo{link to be done with the proposition: division by $\nu$ etc.} 
$\sum_{j=1}^d \ST{\alpha}(x_j/\nu)^2=\sum_{j=1}^d \ST{\alpha}(|x_j|/\nu)^2$ so without loss of generality we assume $x \in \bbR_{+}^{d}$.

The case $\alpha = 0$ and $R = 0$ corresponds to the situation where all $x_j$ are equal to zero or we impose $\nu$ equals to infinity. So we avoid this trivial case.

If $\alpha = 0$ and $R \neq 0$, $\nu = \|x\| / R$. Indeed,
\begin{equation*}
 \sum_{j=1}^d \ST{0}(x_j/\nu)^2 = R^2 \Longleftrightarrow
 \sum_{j=1}^d (x_j/\nu)^2 = R^2 \Longleftrightarrow
 \frac{\|x\|_{2}^{2} }{ \nu^2 } = R^2 \text{ hence the result}.
\end{equation*}
If $\alpha \neq 0$ and $R = 0$, we have :
\begin{equation*}
 \sum_{j=1}^d \ST{\alpha} \left(\frac{x_j}{\nu}\right)^2 = 0
\Longleftrightarrow
 \forall j \in [d], \left(\frac{x_j}{\nu} - \alpha\right)_+ = 0
\Longleftrightarrow
 \forall j \in [d], \frac{x_j}{\nu} \leq \alpha \Longleftrightarrow
 \nu \geq \frac{\max_{j \in [d]} x_j}{\alpha}.
\end{equation*}
So we choose the smallest $\nu$ \ie $\nu = \|x\|_{\infty} / \alpha$. In all the above cases, the computation is done in $O(d)$.

Otherwise $\alpha \neq 0$ and $R \neq 0$.
The function $\nu \mapsto \sum_{j=1}^d \ST{\alpha}(x_j/\nu)^2$ is a
non-increasing continuous function with limit
$+\infty$ (resp. $0$) when $\nu
\to 0$ (resp. $\nu \to +\infty$). Hence, there is a unique solution to
$\sum_{j=1}^d \ST{\alpha}(x_j/\nu)^2 = R^2$.

We denote by $x_{(1)},\dots,x_{(d)}$ the coordinates of $x$ ordered in
decreasing order (with the convention $x_{(0)}=+\infty$ and $x_{(d+1)}=0$). Note that $\sum_{j=1}^d \ST{\alpha}(x_j/\nu)^2 = \sum_{j=1}^d
\ST{\alpha}(x_{(j)}/\nu)^2$. Then, there exists an index $j_0 \in [p]$ such
that
\begin{equation} \label{eq:finding_j_0}
 R^2 \in \left[\sum_{j = 0}^d \ST{\alpha} \left(\alpha \frac{x_{(j)}}{x_{(j_0)}}
\right)^2,\sum_{j = 0}^d \ST{\alpha} \left(\alpha \frac{x_{(j)}}{x_{(j_0+1)}}
\right)^2 \right).
\end{equation}
For such a $j_0$, one can check that $\nu \in
(x_{(j_0+1)}/\alpha,x_{(j_0)}/\alpha]$.
The definition of the soft-thresholding operator yields
\begin{equation} \label{eq:threshold_effect}
\ST{\alpha}(x_j/\nu)^2 =
\begin{cases}(x_j/\nu-\alpha)^2 &
\text{if } x_j \geq \nu\alpha,\\ 0 &\text{if } x_j < \nu \alpha.
\end{cases}
\end{equation}
It can be simplified thanks to $x_j \geq x_{(j_0)} \Rightarrow x_j \geq \nu
\alpha$ and $x_j \leq x_{(j_0+1)}
\Rightarrow x_j < \nu \alpha$.
% $ R^2 = \sum_{j: x_j>\alpha\nu} \ST{\alpha}(x_j/\nu)^2 + 0 $
%
% The intermediate values theorem yields $\nu \in (x_{(j_0+1)},x_{(j_0)})$,
% and the previous equation becomes
% $ R^2 = \sum_{j: x_j\geq \alpha x_{(j_0)}} (x_j/\nu-\alpha)^2. $ \\
% %
% Denoting $S_0:=\{j \in [d] : x_j \geq \alpha x_{(j_0)}\}=\{j \in [d] : x_j >
% \alpha \nu\}$, then
% $
Hence, $R^2 = \sum_{j=1}^{j_0} (x_{(j)}/\nu-\alpha)^2 = \sum_{j=1}^{j_0}
(x_{(j)}/\nu)^2 + \alpha^2\sum_{j=1}^{j_0} 1 - 2\alpha\sum_{j=1}^{j_0}
x_{(j)}/\nu $ so solving $\sum_{j=1}^p \ST{\alpha}(x_{(j)}/\nu)^2 = R^2$
is equivalent to solve on $\bbR_+$
\begin{equation}\label{eq:lambda_equation}
(\alpha^2 j_0-R^2) \nu^2 - \left(2 \alpha\sum_{j=1}^{j_0}
x_{(j)}\right)\nu + \sum_{j=1}^{j_0}  x_{(j)}^2=0.
\end{equation}
If $(\alpha^2 j_0-R^2)=0$, then  $\nu=\sum_{j=1}^{j_0}  x_{(j)}^2/(2\alpha
\sum_{j=1}^{j_0}  x_{(j)})$. Otherwise
$\nu$ is the unique solution lying in
$(x_{(j_0+1)}/\alpha,x_{(j_0)}/\alpha]$ of
the quadratic equation stated in Eq.~\eqref{eq:lambda_equation}.

%  defined as
% \begin{equation}
% \! \!
% \frac{\alpha\displaystyle\sum_{j=1}^{j_0}  x_{(j)} \stackrel{+}{-}
% \sqrt{\alpha^2\left(\displaystyle\sum_{j=1}^{j_0}
% x_{(j)}\right)^2 -(\alpha^2
%j_0-R^2)\displaystyle\sum_{j=1}^{j_0}  x_{(j)}^2 }}{ (\alpha^2 j_0-R^2)}.
% \end{equation}
% Note that the sign $+/-$ depends on whether $\alpha^2 j_0 \geq R^2$ or not.
In the worst case, to compute $\Lambda(x,\alpha,R)$, one needs to sort a vector
of size $d$, what can be done in $O(d\log(d))$ operations, and finding $j_0$
thanks to \eqref{eq:finding_j_0} requires $O(d^2)$ if we apply a naive
algorithm.

In the following, we show that one can easily reduce the complexity to
$O(d\log(d))$ in worst case.
%

% \subsection{Computation of $\Lambda(x,\alpha,R)$}
For all $j$ in $[d], \, \ST{\alpha} \left(\alpha \frac{x_j}{x_{j_0}} \right) =
0$ as soon as $x_j \leq x_{j_0}$. This implies that \eqref{eq:finding_j_0} is
equivalent to
\begin{equation} \label{eq:finding_j_0_cheap}
 R^2 \in \left[\sum_{j = 0}^{j_0 - 1} \ST{\alpha} \left(\alpha
\frac{x_{(j)}}{x_{(j_0)}} \right)^2,
               \sum_{j = 0}^{j_0} \ST{\alpha} \left( \alpha
\frac{x_{(j)}}{x_{(j_0+1)}} \right)^2
         \right).
\end{equation}
Denoting $S_{j_0} := \sum_{j=1}^{j_0} x_{(j)}$ and $S_{j_0}^{(2)} :=
\sum_{j=1}^{j_0} x_{(j)}^{2}$,
a direct calculation show that \eqref{eq:finding_j_0_cheap} can be rewritten as
\begin{equation}
 R^2 \in \alpha^2 \left[ \frac{S_{j_0 - 1}^{(2)}}{x_{(j_0)}^{2}} - 2 \frac{
S_{j_0 - 1}}{x_{(j_0)}} + j_0, \,
		       \frac{S_{j_0}^{(2)}}{x_{(j_0 + 1)}^{2}} - 2 \frac{
S_{j_0}}{x_{(j_0 + 1)}} + j_0 + 1
	        \right).
\end{equation}

Finally, solving $\sum_{j=1}^p \ST{\alpha}(x_{(j)}/\nu)^2 = R^2$ is
equivalent to finding the solution of
$(\alpha^2 j_0 - R^2)\nu^2 - (2\alpha S_{j_0})\nu + S_{j_0}^{(2)} = 0
$ lying in $(x_{(j_0+1)}/\alpha,x_{(j_0)}/\alpha]$. Hence,
\begin{equation} \label{eq:solve_lambda}
\Lambda(x,\alpha,R) = \frac{
\alpha S_{j_0} -
\sqrt{\alpha^2 S_{j_0}^{2} - S_{j_0}^{(2)}(\alpha^2 j_0 - R^2)}
}{\alpha^2 j_0 - R^2} =: \nu_1
\text{ or }
\Lambda(x,\alpha,R) = \frac{
\alpha S_{j_0} +
\sqrt{\alpha^2 S_{j_0}^{2} - S_{j_0}^{(2)}(\alpha^2 j_0 - R^2)}
}{\alpha^2 j_0 - R^2} =: \nu_2.
\end{equation}
%
% Now we show how to choose between these two solutions following the same idea
% as in \citet{Burdakov_Merkulov01}. Let us set
% $H_x(\nu) = \norm{\ST{\alpha \nu}(x)}_2 - \nu R$.
% We have seen that $\Lambda(x,\alpha,R)$ is the solution of
% $H_x(\nu) = 0$ iif it is the solution of
% $ h_x(\nu) = (\alpha^2 j_0 - R^2)\nu^2 - (2\alpha S_{j_0})\nu +
% S_{j_0}^{(2)} = 0$. Then it is easy to see that the stationary point of
% $h_x(\cdot)$, $\nu^\star := \frac{\alpha S_{j_0}}{\alpha^2 j_0 - R^2} <
% \nu_1$.
% Since $H_x(\nu)$ is non-increasing, we deduce that $H_x(\nu_1) <
% H_x(\nu^\star) \neq 0$ because $\nu^\star$ is a not a solution;
% therefore $\nu_1$ is not. \\

\begin{itemize}
 \item If $\alpha^2 j_0 - R^2 < 0$, then $\nu_2 < 0$ and so it cannot be a
solution since $\Lambda(x,\alpha,R)$ must be positive.
  \item Otherwise, we have
  $\nu_2 \geq \frac{\alpha S_{j_0}}{\alpha^2 j_0 - R^2} =
  \frac{1}{\alpha(j_0 - \frac{R^2}{\alpha^2})} \sum_{j=1}^{j_0} x_{(j)} >
  \frac{1}{\alpha j_0} \sum_{j=1}^{j_0} x_{(j)} \geq \frac{x_{(j_0)}}{\alpha}$,
  where the second inequality results from the fact that $j_0 > j_0 -
\frac{R^2}{\alpha^2}$. And again $\nu_2$ cannot be a solution since
$\Lambda(x,\alpha,R)$ belongs to $(x_{(j_0+1)}/\alpha,x_{(j_0)}/\alpha]$.
\end{itemize}
Hence, in all cases, the solution is given by $\nu_1$.

Moreover, we can significantly reduce the cost of the sort. In fact, for all
$\nu$, we have $\norm{\ST{\alpha \nu}(x)} \geq \norm{\ST{\alpha \nu}(x)}_{\infty}
= \max_{j \in [d]} (|x_{j}| - \nu \alpha)_{+}.
$ Hence,
$\norm{\ST{\alpha \nu}(x)} - \nu R \geq
\norm{x}_{\infty} - \nu \alpha - \nu R > 0$ if and only if
$\nu < \frac{\norm{x}_{\infty}}{\alpha + R}$. Combining this with
Equation \eqref{eq:threshold_effect}, we take into account only the coordinates
which have an absolute value greater than $\frac{\alpha \norm{x}_{\infty}}{\alpha +
R}$. Indeed, by contrapositive, if $\nu$ is a solution then $\nu \geq
\frac{\norm{x}_{\infty}}{\alpha + R}$ hence $x_j < \alpha
\frac{\norm{x}_{\infty}}{\alpha + R} \Rightarrow x_j < \nu \alpha
\overset{\eqref{eq:threshold_effect}}{\Rightarrow} \ST{\alpha}(x_j/\nu) =
0$.

Finally, computing $\Lambda(x,\alpha,R)$ can be done by applying Algorithm \ref{alg:sgl_burdakov}. Note that this algorithm is similar to \citep[Algorithm~4]{Burdakov_Merkulov01}.
\end{proof}

\section{Notes about others methods}

% \subsection{TLFre}
% To our knowledge, the first screening method for \SGL is
% the \textbf{TLFre} introduced in \citet{Wang_Ye14}. This sequential rule
% requires the exact knowledge of an dual optimal solution which is not
% accessible in practice. As a consequence, one may discard active variables
% and this may prevent the algorithm from converging as show in
% \citet[figure 4]{Ndiaye_Fercoq_Gramfort_Salmon15} for the Group-Lasso.

\subsection*{Extension of some previous methods to the \SGL}

\subsubsection*{Extension of \citet{ElGhaoui_Viallon_Rabbani12}: \emph{static safe region}}

% We adapt here the safe region from \citet{ElGhaoui_Viallon_Rabbani12}, to the case of \SGL, and we proceed as follows. The equivalent safe region would be $\mathcal{R} = \mathcal{R}_1 \cap \mathcal{R}_2$ where $\mathcal{R}_1$ and $\mathcal{R}_2$ are safe subspaces of $\bbR^n$ corresponding to optimality condition, and they are introduced below:

The \emph{static safe region} can be obtained as in \citep{ElGhaoui_Viallon_Rabbani12} using the ball
$\mathcal{B} \left(\frac{y}{\lambda}, \norm{\frac{y}{\lambda_{\max}} -
\frac{y}{\lambda}} \right)$.

Indeed $y/\lambda_{\max}$ is a dual feasible point. Hence the distance between
$y/\lambda$ and $y/\lambda_{\max}$ is smaller than the distance between $y/\lambda$ and $\ttheta{\lambda,\tau,w}$ since the last point is the projection of $y/\lambda$ over the (close and convex) dual feasible set $\dualomegasgl$.

\subsubsection*{Extension of \citet{Bonnefoy_Emiya_Ralaivola_Gribonval14}: \emph{dynamic safe region}}
The \emph{dynamic safe region} can be obtained as in \citep{ElGhaoui_Viallon_Rabbani12} using the ball $\mathcal{B} \left(\frac{y}{\lambda}, \norm{\theta_k - \frac{y}{\lambda}} \right)$, where the sequence $(\theta_k)_{k \in \bbN}$ converges to the dual optimal vector $\ttheta{\lambda,\tau,w}$.

A sequence of dual points is required to construct such a ball, and following
\citet{ElGhaoui_Viallon_Rabbani12} we can consider the \emph{dual scaling} procedure. We have chosen a simple procedure here: Let $\theta_k = \rho_k/\max(\lambda, \normsgl^D(X^\top \rho_k))$, where $\rho_k:= y-X\beta_k$, for a primal converging sequence $\beta_k$. Hence, one can use the safe sphere $\mathcal{B} \left(\frac{y}{\lambda}, \norm{\theta_k - \frac{y}{\lambda}} \right)$  with the same reasoning as for the \emph{static safe region}.

Hence, we can easily extend the corresponding screening rules to the \SGL
thanks
to the formulation \eqref{eq:safe_rule_feature} and
\eqref{eq:safe_rule_group}.

\subsubsection*{Extension of \citet{Bonnefoy_Emiya_Ralaivola_Gribonval14}: \emph{DST3 safe region}}

Now we show that the safe regions proposed in \citet{Xiang_Xu_Ramadge11} for
the Lasso and generalized in \citet{Bonnefoy_Emiya_Ralaivola_Gribonval14} to the Group-Lasso can be adapted to the \SGL. For that, we define

\begin{equation*}
 \mathcal{V}_{\star} = \left\{ \theta \in \bbR^n:
 \norm{X_{g_{\star}}^{\top}\theta}_{\epsilon_{g_\star}} \leq
 \tau + (1 - \tau) w_{g_\star} \right\} \text{ and }
 \mathcal{H}_{\star} = \bigg\{ \theta \in \bbR^n:
 \langle \theta, \eta \rangle = \tau + (1 - \tau) w_{g_\star} \bigg\}.
\end{equation*}

Where $\eta$ is the vector normal to $\mathcal{V}_{\star}$ at the point
$\frac{y}{\lambda_{\max}}$ and is given by
$\eta:= {X_{g_{\star}}} \nabla \norm{\cdot}_{\epsilon_{g_\star}}
\left(X_{g_{\star}}^{\top}\frac{y}{\lambda_{\max}}\right)$,
where $\nabla \norm{\cdot}_{\epsilon}(x) =
\frac{\ST{(1 - \epsilon)\norm{x}_{\epsilon}}(x)}{
\norm{\ST{(1 - \epsilon)\norm{x}_{\epsilon}}(x)}_{\epsilon}^{D}}$
see Lemma~\ref{lm:gradient_burdakov_norm} below. Let
$\theta_c := \frac{y}{\lambda} -
\left(\frac{\frac{\langle \eta, y \rangle}{\lambda} -
(\tau + (1 - \tau) w_{g_\star})}{\norm{\eta}^{2}}\right) \eta$ be the
projection of $y/\lambda$ onto the hyperplane $\mathcal{H}_{\star}$
and $r_{\theta_k} := \sqrt{\norm{\frac{y}{\lambda} - \theta_k}^{2} -
\norm{\frac{y}{\lambda} - \theta_c}^{2}}$ where $\theta_k$ is a dual feasible vector (which can be obtained by dual scaling). Then, the following proposition holds.
\begin{proposition} Let $\theta_c$ and $r_{\theta_k}$ defined as above. Then
 $\ttheta{\lambda, \tau,w} \in \mathcal{B}(\theta_c, r_{\theta_k})$.
\end{proposition}
%
% We follow and simplify the proof of \citet{Bonnefoy_Emiya_Ralaivola_Gribonval14}.
\begin{proof}
We set $\mathcal{H}_{\star}^{-} := \big\{\theta \in \bbR^n:  \langle \theta,
\eta \rangle
\leq \tau + (1 - \tau) w_{g_\star} \big\}$ the negative half-space induced by the
hyperplane $\mathcal{H}_{\star}$.
Since $\ttheta{\lambda, \tau,w} \in \dualomegasgl \subset \mathcal{V}_{\star} \subset {\mathcal{H}_{\star}^{-}}$ and
$\mathcal{B}\left(\frac{y}{\lambda}, \norm{\frac{y}{\lambda} -
\theta_k}\right)$ is a safe region, then
$\ttheta{\lambda, \tau,w} \in \mathcal{H}_{\star}^{-} \cap
\mathcal{B}\left(\frac{y}{\lambda}, \norm{\frac{y}{\lambda} -
\theta_k}\right)$.

Moreover, for any $\theta \in \mathcal{H}_{\star}^{-} \cap
\mathcal{B}\left(\frac{y}{\lambda}, \norm{\frac{y}{\lambda} -
\theta_k}\right)$, we have: %\jo{we have to state that the half-space is safe}
\begin{equation*}
 \norm{\frac{y}{\lambda} - \theta_k}^{2} \geq
\norm{\frac{y}{\lambda} - \theta}^{2} =
\norm{\left(\frac{y}{\lambda} - \theta_c\right) + \left(\theta_c - \theta\right)}^{2} =
\norm{\frac{y}{\lambda} - \theta_c}^{2} + \norm{\theta_c - \theta}^{2} +
2 \left\langle \frac{y}{\lambda} - \theta_c, \theta_c - \theta \right\rangle.
\end{equation*}
Since $\theta_c = \Pi_{\mathcal{H}_{\star}^{-}}(\frac{y}{\lambda})$ and
$\mathcal{H}_{\star}^{-}$ is convex,
then $\langle \theta_c - \frac{y}{\lambda}, \theta_c - \theta \rangle \leq 0$. Thus
\begin{equation*}
 \norm{\frac{y}{\lambda} - \theta_k}^{2} \geq
 \norm{\frac{y}{\lambda} - \theta_c}^{2} + \norm{\theta_c - \theta}^{2},
 \text{ hence }
 \norm{\theta - \theta_c} \leq
 \sqrt{\norm{\frac{y}{\lambda} - \theta_k}^{2} -
       \norm{\frac{y}{\lambda} - \theta_c}^{2}} =: r_{\theta_k}.
\end{equation*}

Which show that $\mathcal{H}_{\star}^{-} \cap
\mathcal{B}\left(\frac{y}{\lambda}, \norm{\frac{y}{\lambda} -
\theta_k}\right) \subset \mathcal{B}(\theta_c, r_{\theta_k})$. Hence the result.
\end{proof}

\section{\SGL plus Elastic Net}
The Elastic-Net estimator (\citet{Zou_Hastie05}) can be mixed with the \SGL by considering
\begin{equation} \label{eq:elastic_net_sgl}
\argmin_{\beta \in \bbR^p} \,
\frac{1}{2} \norm{y - X\beta}^2 + \lambda_1 \Omega(\beta) + \frac{\lambda_2}{2}
\norm{\beta}^2.
\end{equation}
By setting
$\tilde{X}=\begin{pmatrix}
X\\
\sqrt{\lambda_2} \Id_p
\end{pmatrix}\in \bbR^{n+p,p}
$
and
$\tilde{y}=\begin{pmatrix}
y\\
0
\end{pmatrix} \in \bbR^{n+p}
$,
we can reformulate \eqref{eq:elastic_net_sgl} as
\begin{equation}
\argmin_{\beta \in \bbR^p} \,
\frac{1}{2} \norm{\tilde{y} - \tilde{X}\beta}^2 + \lambda_1 \Omega(\beta),
\end{equation}
and we can adapt our GAP safe rule framework to this case.

\section{Properties of the $\epsilon$-norm}
% \section*{Dual norm, gradient and characterization of the unit ball of the
% $\epsilon$-norm}

We describe, for completeness, some properties of the
$\epsilon$-norm. The following materials are from \cite{Burdakov_Merkulov01}
with some adaptations.

\begin{lemma} \label{lm:epsilon_decomposition}
For all $\xi \in \bbR^d$, the $\epsilon$-decomposition reads:
\begin{align*}
 \xi &= \xi^{\epsilon} + \xi^{1 - \epsilon} \text{ where }
 \xi^{\epsilon} := \ST{(1 - \epsilon)\norm{\xi}_{\epsilon}} (\xi) \text{ and }
 \xi^{1 - \epsilon} := \xi - \xi^{\epsilon}. \\
 \norm{\xi^{\epsilon}} &= \epsilon \norm{\xi}_{\epsilon} \text{ and }
 \norm{\xi^{1 - \epsilon}}_{\infty} = (1 - \epsilon) \norm{\xi}_{\epsilon}.
 \text{ Hence }
 \norm{\xi}_{\epsilon} = \norm{\xi^{\epsilon}} + \norm{\xi^{1 - \epsilon}}_{\infty}.
\end{align*}
\end{lemma}

\begin{proof}
$\norm{\xi^{\epsilon}} = \norm{\ST{(1 - \epsilon)\norm{\xi}_{\epsilon}} (\xi)} =
\epsilon \norm{\xi}_{\epsilon}$ by definition of the $\epsilon$-norm
$\norm{\xi}_{\epsilon}$. Moreover,
\begin{align*}
\xi^{1 - \epsilon} &=
\sum_{i=1}^{d} \left[\xi_i - \sign(\xi_i)(|\xi_i| -
(1 - \epsilon)\norm{\xi}_{\epsilon})_{+} \right] =
\sum_{i=1}^{d} \sign(\xi_i) \left[|\xi_i| - (|\xi_i| -
(1 - \epsilon) \norm{\xi}_{\epsilon})_{+} \right]. \text{ Thus,}\\
% \begin{align*}
\norm{\xi^{1 - \epsilon}}_{\infty} &=
 \max_{i \in [d]} \left|\sign(\xi_i) \left[|\xi_i| - (|\xi_i| -
(1 - \epsilon) \norm{\xi}_{\epsilon})_{+} \right] \right| \\
& = \max \big(
 \max_{\underset{|\xi_i| \leq (1 - \epsilon) \norm{\xi}_{\epsilon}}{i \in [d]}}
\left||\xi_i| - (|\xi_i| -
(1 - \epsilon) \norm{\xi}_{\epsilon})_{+} \right|,
 \max_{\underset{|\xi_i| > (1 - \epsilon) \norm{\xi}_{\epsilon}}{i \in [d]}}
\left||\xi_i| - (|\xi_i| -
(1 - \epsilon) \norm{\xi}_{\epsilon})_{+}\right| \big) \\
&= \max \big(
\max_{\underset{|\xi_i| \leq (1 - \epsilon) \norm{\xi}_{\epsilon}}{i \in [d]}}|\xi_i|,
\, (1 - \epsilon) \norm{\xi}_{\epsilon} \big) =
(1 - \epsilon) \norm{\xi}_{\epsilon}.
\end{align*}
\end{proof}

\begin{lemma} \label{lm:computation_of_epsilon_decomposition}
Let us define
$U(\norm{\xi}_{\epsilon}) := \{u \in \bbR^d: \norm{u} \leq \epsilon
\norm{\xi}_{\epsilon}\}$ and
$V(\norm{\xi}_{\epsilon}) := \{v \in \bbR^d: \norm{v}_{\infty} \leq (1 -
\epsilon)\norm{\xi}_{\epsilon}\}$. Then
\begin{equation*}
\xi^{(1 - \epsilon)} = \argmin_{\underset{\xi = u + v }{u \in
U(\norm{\xi}_{\epsilon})}} \norm{v}_{\infty} \text{ and }
\xi^{\epsilon} = \argmin_{\underset{\xi = u + v}{v \in
V(\norm{\xi}_{\epsilon})}} \norm{u}.
\end{equation*}
\end{lemma}

\begin{proof}
\begin{itemize}
 \item \textbf{Existence and uniqueness of the solutions} \\
It is clear that
$\displaystyle \argmin_{\underset{\xi = u + v }{u \in
U(\norm{\xi}_{\epsilon})}} \norm{v}_{\infty} =
\argmin_{\xi - U(\norm{\xi}_{\epsilon})} \norm{v}_{\infty}$
and
$\displaystyle \argmin_{\underset{\xi = u + v}{v \in
V(\norm{\xi}_{\epsilon})}} \norm{u} =
\argmin_{\xi - V(\norm{\xi}_{\epsilon})} \norm{u}$. Thus, these two
problems have unique solution because we minimize strict convex functions onto convex
sets.
 \item \textbf{Uniqueness of the $\epsilon$-decomposition} \\
 From Lemma~\ref{lm:epsilon_decomposition} we have
 $\xi = \xi^{\epsilon} + \xi^{1 - \epsilon}$ where
 $\norm{\xi^{\epsilon}} = \epsilon \norm{\xi}_{\epsilon}$ and
 $\norm{\xi^{1 - \epsilon}}_{\infty} = (1 - \epsilon) \norm{\xi}_{\epsilon}$.
 Hence $\xi^{\epsilon} \in U(\norm{\xi}_{\epsilon})$ and
 $\xi^{1 - \epsilon} \in V(\norm{\xi}_{\epsilon})$. Now it suffices to show that
this $\epsilon$-decomposition is unique. \\

Suppose $\xi \neq 0$ (the uniqueness is trivial otherwise) and $v \in
V(\norm{\xi}_{\epsilon})$. We show that for any $u \in \bbR^d$ such that
$\xi = u + v, \, v \neq \xi^{1 - \epsilon}$ implies
$u \notin U(\norm{\xi}_{\epsilon})$. \\

\begin{equation*}
\norm{u}^2 = \norm{\xi - v}^2 = \norm{\xi^{\epsilon} + (\xi^{1 - \epsilon} - v)}^2
= \norm{\xi^{\epsilon}}^2 + 2 \langle{\xi^{\epsilon}}, \xi^{1 - \epsilon} - v\rangle +
\norm{\xi^{1 - \epsilon} - v}^2,
\end{equation*}
hence
$\norm{u}^2 > \epsilon^2 \norm{\xi}_{\epsilon}^{2} +
2 \langle \xi^{\epsilon}, \, \xi^{1 - \epsilon} - v \rangle$ because
$\norm{\xi^{\epsilon}} = \epsilon \norm{\xi}_{\epsilon}$ and
$\norm{\xi^{1 - \epsilon} - v} > 0$ ($v \neq \xi^{1 - \epsilon}$). Moreover,
\begin{align*}
\langle \xi^{\epsilon}, \, \xi^{1 - \epsilon} - v \rangle &=
\sum_{i = 1}^{d}
\left[\sign(\xi_i)(|\xi_i| - (1 - \epsilon)\norm{\xi}_{\epsilon})_{+} \right]
\left[\sign(\xi_i)(|\xi_i| - (|\xi_i| -
(1 - \epsilon) \norm{\xi}_{\epsilon})_{+}) - v_i \right] \\
&= \sum_{i = 1}^{d}
\left[(|\xi_i| - (1 - \epsilon)\norm{\xi}_{\epsilon})_{+} \right]
\left[(|\xi_i| - (|\xi_i| -
(1 - \epsilon) \norm{\xi}_{\epsilon})_{+}) - v_i \sign(\xi_i) \right] \\
&\geq \sum_{\underset{|\xi_i| > (1 - \epsilon)\norm{\xi}_{\epsilon}}{i = 1}}
\left[|\xi_i| - (1 - \epsilon)\norm{\xi}_{\epsilon} \right]
\left[(1 - \epsilon) \norm{\xi}_{\epsilon} - v_i \sign(\xi_i) \right] \geq 0.
\end{align*}
The last inequality hold because $v \in V(\norm{\xi}_{\epsilon})$ \ie $\forall i
\in [d], \, v_i \leq (1 - \epsilon) \norm{\xi}_{\epsilon}$. Finally, $\norm{u}^2
> \epsilon^2 \norm{\xi}_{\epsilon}^{2}$ hence the result.
\end{itemize}
\end{proof}

\begin{lemma} \label{lm:epsilon_ball}
$  \left\{ \xi \in \bbR^d: \norm{\xi}_{\epsilon} \leq \nu \right\} =
  \left\{ u + v: u, v \in \bbR^d, \norm{u} \leq \epsilon \nu,
  \norm{v}_{\infty} \leq (1 - \epsilon) \nu \right\}.
$
\end{lemma}
\begin{proof}
Thanks to Lemma~\ref{lm:epsilon_decomposition},
we have $\xi = \xi^{\epsilon} + \xi^{1 - \epsilon}$,
$\norm{\xi^{\epsilon}} = \epsilon \norm{\xi}_{\epsilon}$ and
$\norm{\xi^{1 - \epsilon}}_{\infty} = (1 - \epsilon) \norm{\xi}_{\epsilon}$. Hence,
$\norm{\xi}_{\epsilon} \leq \nu$ implies
$\norm{\xi^{\epsilon}} \leq \epsilon \nu$ and
$\norm{\xi^{1 - \epsilon}}_{\infty} \leq (1 - \epsilon) \nu$. \\

Suppose $\xi = u + v$ such that $\norm{u} \leq \epsilon \nu$ and
$\norm{v}_{\infty} \leq (1 - \epsilon) \nu$. From the $\epsilon$-decomposition,
we have
$\norm{\xi}_{\epsilon} = \norm{\xi^{\epsilon}} + \norm{\xi^{1 - \epsilon}}_{\infty}$.
Moreover, $\norm{\xi^{\epsilon}} \leq \norm{u}$ and
$\norm{\xi^{1 - \epsilon}}_{\infty} \leq \norm{v}_{\infty}$ thanks to Lemma~\ref{lm:computation_of_epsilon_decomposition}. Hence
$\norm{\xi^{\epsilon}} \leq \norm{u} + \norm{v}_{\infty} \leq
\epsilon \nu + (1 - \epsilon) \nu = \nu.$
\end{proof}

\begin{lemma}[Dual norm of the $\epsilon$-norm] \label{lm:epsilon_dual_norm}
Let $\xi \in \bbR^d$, then
$\norm{\xi}_{\epsilon}^{D} = \epsilon \norm{\xi} + (1 - \epsilon)
\norm{\xi}_1$.
\end{lemma}

\begin{proof}
 \begin{align*}
\norm{\xi}_{\epsilon}^{D} &:= \max_{\norm{x}_{\epsilon} \leq 1} \xi^{\top} x
= \max_{\underset{\norm{v}_{\infty} \leq 1 - \epsilon}{\norm{u}
\leq \epsilon}} \xi^{\top} (u + v) \text{ thanks to Lemma~\ref{lm:epsilon_ball}}
 \\
&= \epsilon \max_{\norm{u} \leq 1} \xi^{\top} u +
   (1 - \epsilon) \max_{\norm{v}_{\infty} \leq 1} \xi^{\top} v
= \epsilon \norm{\xi}^D + (1 - \epsilon) \norm{\xi}_{\infty}^{D}.\qedhere
\end{align*}
\end{proof}

\begin{lemma} \label{lm:gradient_burdakov_norm}
Let $\xi \in \bbR^d \backslash \{0\}$. Then
 $\nabla \norm{\cdot}_{\epsilon}(\xi) =
\frac{\xi^{\epsilon}}{\norm{\xi^{\epsilon}}_{\epsilon}^{D}}$.
\end{lemma}

\begin{proof}
 Let us define $h: \bbR \times \bbR^d \mapsto \bbR$ by
 $h(\nu, \xi) = \norm{\ST{(1 - \epsilon)\nu}(\xi)} - \epsilon \nu$. Then
 we have
 \begin{align*}
\frac{\partial h}{\partial \nu} (\nu, \xi) &=
 \frac{{\ST{(1 - \epsilon)\nu}(\xi)}^{\top}}{\norm{\ST{(1 - \epsilon)\nu}(\xi)}}
 \frac{\partial \ST{(1 - \epsilon)\nu}(\xi)}{\partial \nu} - \epsilon =
 - \frac{{\ST{(1 - \epsilon)\nu}(\xi)}^{\top}}{\norm{\ST{(1 - \epsilon)\nu}(\xi)}}
 (1 - \epsilon) \sign(\xi) - \epsilon \\
&= - \frac{\norm{\ST{(1 - \epsilon)\nu}(\xi)}_1}{\norm{\ST{(1 - \epsilon)\nu}(\xi)}}
 (1 - \epsilon) - \epsilon =
- \frac{(1 - \epsilon) \norm{\ST{(1 - \epsilon)\nu}(\xi)}_1 +
\epsilon \norm{\ST{(1 - \epsilon)\nu}(\xi)}}
{\norm{\ST{(1 - \epsilon)\nu}(\xi)}} \\
&= - \frac{\norm{\ST{(1 - \epsilon)\nu}(\xi)}_{\epsilon}^{D}}
{\norm{\ST{(1 - \epsilon)\nu}(\xi)}} \text{ thanks to Lemma~\ref{lm:epsilon_dual_norm}}.
 \end{align*}

By definition of the $\epsilon$-norm, $h(\norm{\xi}_{\epsilon}, \xi) = 0$. Since
$\frac{\partial h}{\partial \nu} (\norm{\xi}_{\epsilon}, \xi) =
- \frac{\norm{\xi^{\epsilon}}_{\epsilon}^{D}}{\epsilon \norm{\xi}_{\epsilon}} \neq 0$,
we obtain by applying the Implicit Function Theorem
\begin{equation*}
\nabla \norm{\cdot}_{\epsilon}(\xi) \times \frac{\partial h}{\partial \nu}
(\norm{\xi}_{\epsilon}, \xi) +
\frac{\partial h}{\partial \xi} (\norm{\xi}_{\epsilon}, \xi) = 0 \text{ hence }
\nabla \norm{\cdot}_{\epsilon}(\xi) = - \frac{\frac{\partial h}{\partial \xi}
(\norm{\xi}_{\epsilon}, \xi)}{\frac{\partial h}{\partial \nu}
(\norm{\xi}_{\epsilon}, \xi)}.
\end{equation*}
Moreover,
$\frac{\partial h}{\partial \xi}(\norm{\xi}_{\epsilon}, \xi) =
\frac{\ST{(1 - \epsilon)\norm{\xi}_{\epsilon}}(\xi)}{\norm{\ST{(1 -
\epsilon)\norm{\xi}_{\epsilon}}(\xi)}} = \frac{\xi^{\epsilon}}{\norm{\xi^{\epsilon}}} =
\frac{\xi^{\epsilon}}{\epsilon \norm{\xi}_{\epsilon}}$ hence the result:
$\nabla \norm{\cdot}_{\epsilon}(\xi) =
\frac{\xi^{\epsilon}}{\norm{\xi^{\epsilon}}_{\epsilon}^{D}}$.
\end{proof}

% \subsection{Derivation of the dual formulation : }
% % theorem
% % \ref{th:general_primal_dual}}
% \eugene{TODO add it}

% \input{subfiles/img}
% \input{subfiles/multitask}

\end{document}